%% file: 24-1466.tex
\documentclass[twoside,11pt]{article}

\usepackage[nohyperref]{jmlr2e} 
\usepackage[hidelinks]{hyperref}

\usepackage{microtype}
\usepackage{graphicx}
\usepackage{booktabs}
\usepackage{xcolor}
\usepackage{url}
\usepackage{nicefrac}
\usepackage{algorithm2e}
\usepackage{comment}
\usepackage{enumitem}
\usepackage{tikz}
\usepackage{tabularx}
\usepackage{multirow}
\usepackage{subcaption}
\usepackage{placeins}
\usepackage[makeroom]{cancel}
\usepackage{dsfont}
\usepackage{adjustbox}
\usepackage{array}
\usepackage{amsmath}
\usepackage{amssymb}
\let\proof\relax
\let\endproof\relax
\usepackage{amsthm}
\usepackage{mathtools}
\usepackage{amsfonts}
\usepackage{bbm}
\usepackage{bm}
\usepackage{thmtools}
\usepackage{thm-restate}

\usepackage[capitalize,noabbrev]{cleveref}
\usepackage{amsmath}
\DeclareMathOperator*{\argmax}{arg\,max}
\DeclareMathOperator*{\argmin}{arg\,min}
\newcommand{\algnamegen}{\texttt{R-$\square$-UCB}\@\xspace}

\newcommand{\algnamegenshort}{\texttt{R-$\square$-UCB}\@\xspace}
\newcommand{\algnameblock}{\texttt{DR-BD-UB}\@\xspace}
\newcommand{\algnameblockshort}{\texttt{DR-BD-UB}\@\xspace}
\newcommand{\algnamegendet}{\texttt{DR-G-UB}\@\xspace}
\newcommand{\algnamegendetshort}{\texttt{DR-G-UB}\@\xspace}
\newcommand{\rawucb}{\texttt{RAW-UCB}\@\xspace}

\newcommand{\settingnameshort}{{GTB}\@\xspace}

\newcommand{\vnu}{\bm{\nu}}
\newcommand{\Nonehalf}{\mathbb{E}_{\vnu} \left[ N_1^R \left( \frac{T}{2} \right) \right]}
\newcommand{\Ntwohalf}{\mathbb{E}_{\vnu} \left[ N_2^R \left( \frac{T}{2} \right) \right]}
\newcommand{\NtwoT}{\mathbb{E}_{\vnu} \left[ N_2^R \left( T \right) \right]}
\newcommand{\Nthreehalf}{\mathbb{E}_{\vnu} \left[ N_3^R \left( \frac{T}{2} \right) \right]}
\newcommand{\NthreeT}{\mathbb{E}_{\vnu} \left[ N_3^R \left( T \right) \right]}

\DeclareRobustCommand{\eg}{e.g.,\@\xspace} 
\DeclareRobustCommand{\ie}{i.e.,\@\xspace}
\DeclareRobustCommand{\wrt}{w.r.t.\@\xspace}

\DeclareRobustCommand{\quotes}[1]{``#1''}

\newenvironment{proofsketch}{%
  \proof}{\endproof}

\definecolor{blued}{RGB}{70,197,221}
\definecolor{applegreen}{rgb}{0.55, 0.71, 0.0}
\definecolor{brightBlue}{RGB}{68, 119, 170}
\definecolor{brightCyan}{RGB}{102, 204, 238}
\definecolor{brightGreen}{RGB}{34, 136, 51}
\definecolor{brightYellow}{RGB}{204, 187, 68}
\definecolor{brightRed}{RGB}{238, 102, 119}
\definecolor{brightPurple}{RGB}{170, 51, 119}
\definecolor{brightGrey}{RGB}{187, 187, 187}
\definecolor{vibrantBlue}{RGB}{0, 119, 187}
\definecolor{vibrantCyan}{RGB}{51, 187, 238}
\definecolor{vibrantTeal}{RGB}{0, 153, 136}
\definecolor{vibrantOrange}{RGB}{238, 119, 51}
\definecolor{vibrantRed}{RGB}{204, 51, 17}
\definecolor{vibrantMagenta}{RGB}{238, 51, 119}
\definecolor{vibrantGrey}{RGB}{100, 100, 100}

\usepgflibrary{shadings}

\declaretheorem[name=Proposition,sharenumber=thr]{proposition}
\declaretheorem[name=Assumption]{ass}
\declaretheorem[name=Lemma,sharenumber=thr]{lemma}
\declaretheorem[name=Remark]{remark}

\usepackage{lastpage}
\jmlrheading{27}{2026}{1-\pageref{LastPage}}{9/24; Revised
5/26}{6/26}{24-1466}{G. Genalti, M. Mussi, N. Gatti, M. Restelli, M. Castiglioni and A. M. Metelli}
\ShortHeadings{Graph-Triggered Bandits}{Genalti, Mussi, Gatti, Restelli, Castiglioni and Metelli}
\firstpageno{1}

\allowdisplaybreaks[4]

\begin{document}

\title{Bridging Rested and Restless Bandits with Graph-Triggering: Rising and Rotting}

\author{\name Gianmarco Genalti \email gianmarco.genalti@polimi.it \\
\name Marco Mussi \email marco.mussi@polimi.it \\
\name Nicola Gatti \email nicola.gatti@polimi.it \\
\name Marcello Restelli \email marcello.restelli@polimi.it \\
\name Matteo Castiglioni \email matteo.castiglioni@polimi.it \\
\name Alberto Maria Metelli \email albertomaria.metelli@polimi.it \\
\addr Politecnico di Milano \\
Piazza Leonardo da Vinci 32, Milan, 20133, Italy}

\editor{Kevin Jamieson}

\maketitle

\begin{abstract}%
Rested and Restless Bandits are two well-known bandit settings that are useful to model real-world sequential decision-making problems in which the expected reward of an arm evolves over time due to the actions we perform or due to the nature. In this work, we propose Graph-Triggered Bandits (GTBs), a unifying framework to generalize and extend rested and restless bandits. In this setting, the evolution of the arms' expected rewards is governed by a graph defined over the arms. An edge connecting a pair of arms $(i,j)$ represents the fact that a pull of arm $i$ triggers the evolution of arm $j$, and vice versa. Interestingly, rested and restless bandits are both special cases of our model for some suitable (degenerated) graph. As relevant case studies for this setting, we focus on two specific types of monotonic bandits: rising, where the expected reward of an arm grows as the number of triggers increases, and rotting, where the opposite behavior occurs. For these cases, we study the optimal policies. We provide suitable algorithms for all scenarios and discuss their theoretical guarantees, highlighting the complexity of the learning problem concerning instance-dependent terms that encode specific properties of the underlying graph structure.\footnote{A conference version of this work~\citep{genalti2024graph}, studying Rising GTBs only, appeared at the \emph{International Conference on Machine Learning}.}
\end{abstract}


\input{content/01introduction}
\input{content/02problemformulation}
\input{content/03rising}
\input{content/04rotting}
\input{content/05conclusions}

\setlength{\parindent}{0pt}

\appendix

\input{content/06apx_related}
\input{content/07apx_rising}
\input{content/08apx_rotting}

\bibliography{biblio}

\end{document}

%% file: content/01introduction.tex
\section{Introduction}
\label{sec:intro}

In the basic stochastic Multi-Armed Bandit~\citep[MAB,][]{lattimore2020bandit} problem, at each round, the learner is asked to choose an action (a.k.a.~arm) among a finite action set and, then, it observes a reward drawn from an unknown probability distribution. The simplicity of the MAB framework is both a strength and a limitation. On the one hand, the simple nature of the framework allows for the development of elegant and efficient algorithms that can be exactly characterized and studied from an information-theoretic perspective. On the other hand, the basic MAB model assumes a relatively simplistic environment that may not capture the complexities of real-world situations. As a result, traditional MAB approaches might not be sufficient for more intricate decision-making problems where additional factors come into play. 
To address these limitations, researchers extended the MAB framework by incorporating additional structures and complexities in order to be able to handle realistic scenarios. Examples of that are \emph{linear}~\citep{abbasi2011improved}, \emph{continuous-action spaces}~\citep{kleinberg2008multi}, and \emph{kernelized} bandits~\citep{chowdhury2017kernelized}, which impose structure over the arms, \emph{non-stationary} bandits~\citep{gur2014stochastic}, which allow us to consider evolving environments, \emph{delayed} reward bandits~\citep{pikeburke2018bandits}, allowing us to consider delayed feedback.
Over the different structures available in the literature, we focus on these two specific types of MAB structures, called \emph{restless} and \emph{rested} bandits~\citep {tekin2012online}. In the former, the expected rewards evolve following the time (\ie as an effect of \emph{nature}); in the latter, the expected reward of an arm evolves as a function of the pulls we perform on that specific arm.

In this paper, we propose a unified framework to generalize restless and rested bandits. In particular, we define a novel space of MABs called Graph-Triggered Bandits (GTBs). A GTB is represented by a bandit complemented with a \emph{graph} describing the interactions between arms. Specifically, an arm \emph{triggers} the evolution of its own expected reward (as for rested bandits) and the evolution of the \quotes{connected} arms. Figure~\ref{fig:examples_restless_rested} shows an example of this scenario, where the nodes represent arms, and the edges represent interactions. Interestingly, rested and restless bandits are two vertices in the space of GTBs. In particular, restless bandits correspond to the case of a \emph{fully-connected} graph, while rested ones correspond to the graph with the \emph{self-loops only}. 

This framework is driven by both \emph{theoretical} and \emph{practical} considerations. Theoretically, it offers a unified approach that generalizes \emph{both rested and restless} bandits. Specifically, our goal is to establish a framework in which the well-known rested and restless bandits emerge as special cases, represented by appropriate (degenerated) graphs (see Figure~\ref{fig:examples_restless_rested}). Practically, restless and rested bandits can model a wide range of real-world situations. For example, consider the scenario where we must choose which product to advertise (represented by our arms), with the reward being the number of sales for that product. On the one hand, with rested bandits, we can handle cases in which the products are all independent. On the other hand, with restless bandits, we can handle scenarios in which all the products interact. However, all the intermediate scenarios, \eg where advertising a product boosts its sales and also enhances the sales of the subset of complementary products, cannot be handled using restless/rested solutions. Indeed, this scenario is a rested problem with elements exhibiting restless behavior, and our generalization allows us to address such situations.

\begin{figure}[t!]
    \centering
    \resizebox{0.95\linewidth}{!}{\includegraphics{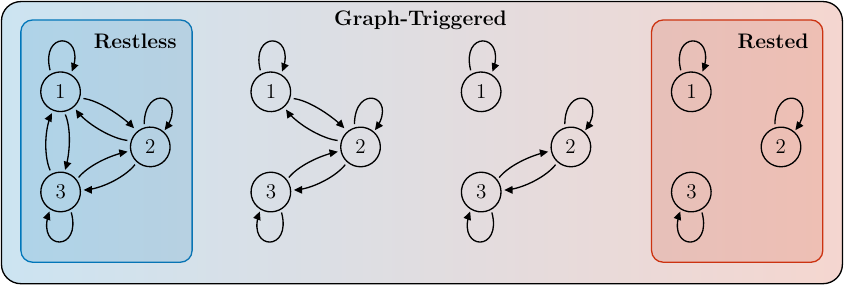}}
    \caption{Examples of $3$-armed GTBs.}
    \label{fig:examples_restless_rested}
\end{figure}

\paragraph{Contributions.}
In this paper, we present Graph-Triggered Bandits (GTBs), a setting aiming to generalize and extend rested and restless bandits settings by introducing a graph structure to represent the interaction between the arms. We focus on the cases of rising and rotting bandits, as they represent interesting case studies allowing us to obtain no-regret algorithms. More in detail, the contributions are as follows.
\begin{itemize}[noitemsep,leftmargin=.5cm,topsep=0pt]
    \item In Section~\ref{sec:problemformulation}, after having introduced the fundamental notions on the rested and restless bandits, we introduce the novel framework of GTBs and discuss the relevant quantities characterizing an instance, including a representation of the graph based on the connectivity matrix. Then, we present the learning problem and the performance index to evaluate algorithms in this setting. 
    \item In Section~\ref{sec:rising}, we study the \emph{Rising GTBs} scenario. We discuss the optimal policy in this setting, by first providing a negative result, showing that computing the optimal policy is NP-hard for an arbitrary graph (Theorem~\ref{thr:hardness}). Then, we characterize the optimal policy for \emph{block-diagonal} connectivity matrices, which can be computed in polynomial time (Theorem~\ref{thr:optimalblocks}). 
    Subsequently, we discuss the deterministic scenario, and we propose two algorithms, the first, \algnameblock, for block-diagonal connectivity matrices and the second, \algnamegendet, for general graphs. We analyze their regret guarantees, highlighting the dependence on the graph structure (Theorems~\ref{thr:det_block_bound} and~\ref{thr:regret_det_gen}). Finally, we analyze the \algnamegenshort algorithm~\citep{metelli2022stochastic}, designed for rested and restless stochastic rising bandits that does not require the knowledge of the graph. We characterize its regret guarantees, focusing on the dependence on the characteristics of the underlying graph (Theorems~\ref{thr:regret_block} and~\ref{thr:regret_gen}). In particular, we show that the regret of \algnamegenshort can be controlled regardless of the nature of the connectivity matrix of the rising GTB instance. In benign instances in which the total increment of the reward functions is sublinear in $T$, the \algnamegenshort achieves sublinear regret.
    \item In Section~\ref{sec:rotting}, we study the \emph{Rotting GTBs} scenario. As for the Rising GTBs case, we prove that computing the optimal policy is NP-hard for arbitrary graphs (Theorem~\ref{thr:hardnessRotting}). Then, we characterize the optimal policy for \emph{block-diagonal} connectivity matrices, which admits a convenient closed-form solution (Theorem~\ref{thr:opt}).
    Then, we focus on the special case of block-diagonal connectivity matrices, and we study how the \rawucb algorithm~\citep{seznec2020single} obtains strong regret guarantees with no knowledge of the graph (Theorem~\ref{thr:regret_block_ROT}). Finally, we present a non-learnability result for \emph{all} the Rotting GTBs problem under general matrices (Theorem~\ref{thr:rottingregretLBgeneric}).
\end{itemize}
The relevant literature is discussed in Appendix~\ref{sec:related}. The proofs of all the statements are provided in Appendices~\ref{apx:proofs} and~\ref{apx:rotting} for the Rising and Rotting GTBs, respectively.\footnote{For Rising GTBs, we report a short version of the proofs. The extended version is provided in~\citep{genalti2024graph}.}

%% file: content/02problemformulation.tex
\section{Graph-Triggered Bandits}
\label{sec:problemformulation}
In this section, we present the framework of Graph-Triggered Bandits (GTBs). We start in Section~\ref{sec:problemformulation:rested-ess} by introducing the basic background notions on stochastic rested and restless bandits. Then, in Section~\ref{sec:problemformulation:setting}, we formalize the GTBs setting. Finally, in Section~\ref{sec:problemformulation:learning}, we formalize the learning problem for the GTBs setting.

\subsection{Notions on Rested and Restless Bandits}
\label{sec:problemformulation:rested-ess}
Let $T\in\mathbb{N}$ be the learning horizon. We define an instance $\boldsymbol{\nu}=(\nu_i)_{i\in [k]}$ of a $k$-armed bandit as a vector of probability distributions with support defined over $\mathbb{R}$, where $k\in\mathbb{N}$.\footnote{Given $k\in\mathbb{N}$, we define $[k] \coloneqq \{1, 2, \ldots,k\}$.} The agent interacts with the environment as follows. At every round $t\in[T]$, the agent is asked to select an action $I_t$ among the $k$ available ones and it observes a reward $X_{I_t,t} \sim \nu_{I_t}$. We define $N_{i,t} \coloneqq \sum_{\tau \in [t]} \mathbbm{1}\{ I_\tau = i \} $ as the number of pulls of the arm $i\in[k]$ until round $t$. 
We consider two specific types of MAB, namely \emph{restless} and \emph{rested} bandits~\citep{tekin2012online}. In both cases, to each arm $i \in[k]$ corresponds a sequence of probability distributions $\boldsymbol{\nu}=(\nu_{i,n})_{i \in [k],\, n \in [T]}$, where the expected reward $\mu_i(n) = \mathbb{E}_{X \sim \nu_{i,n}}[X]$ evolves according to an history-dependent quantity $n \in \mathbb{N}$.
In the rested scenario, the expected reward of a generic arm $i$ evolves according to the number of pulls of such an arm, \ie $n  \leftarrow N_{i,t}$. Conversely, in the restless case, the expected reward of a generic arm $i$ evolves according to the current time $t$, \ie $n \leftarrow t$. This means that, in rested bandits, the reward distribution of an arm evolves only when it is pulled, while in restless bandits, it evolves at each round, no matter the action performed.
As customary in this field, we consider expected rewards $\mu_i(n)$ bounded in $[0,1]$, for every $i\in[k]$ and $n\in [T]$. Finally, we assume distributions to be \emph{subgaussian}\footnote{A (zero-mean) random variable $X$ is $\sigma^2$-subgaussian if it holds $\mathbb{E}\left[ \exp \left( \lambda X \right) \right] \le \exp \left( \frac{\sigma^2 \lambda^2}{2} \right)$ for every $\lambda \in \mathbb{R}$.}  for every arm $i$ and $n \in \mathbb{N}$, with their subgaussianity constants upper bounded by $\sigma^2$.

\subsection{Setting}
\label{sec:problemformulation:setting}
In rested and restless bandits, there exists no structure among different arms. We now present a generalization of rested and restless bandits obtained by adding a structure allowing arms to interact. 
We consider arms as connected through an undirected graph, which can be either \emph{known} or \emph{unknown} to the agent. If we pull an arm $i\in[k]$, we get its reward, and we \emph{trigger} an evolution of the expected reward of the arm $i$ and of all the arms connected to $i$. We neither get nor observe rewards from the connected arms (\ie bandit feedback).
Such a graph can be represented by a symmetric Connectivity Matrix (CM) $\mathbf{G}\in\{0,1\}^{k\times k}$. If the matrix contains the value $1$ in row $i$ and column $j$, this implies that the pull of arm $i$ determines the evolution of the expected reward of arm $j$. If the matrix contains a $0$ in position $(i,j)$, this implies that a pull of arm $i$ does not cause an evolution of the expected reward of arm $j$. The pull of an arm $i$ always implies the evolution of its own expected reward, formally $\mathbf{G}_{i,i}=1, \; \forall i \in [k]$.
For every round $t \in [T]$ and arm $i \in [k]$, we define the number $\widetilde{N}_{i,t}$ of \emph{triggers} that it has undergone as follows:
\begin{align}\label{eq:def_ntilde}
\widetilde{N}_{i,t} =  \sum_{\tau \in [t]} \mathds{1}\{\mathbf{G}_{I_\tau,i} = 1 \} =\mathbf{e}_i^\top \mathbf{G}^\top \mathbf{N}_{t},
\end{align}
where $\mathbf{e}_i$ is a vector belonging to the canonical basis of $\mathbb{R}^k$ whose components are all zero except for the $i$-th and $\mathbf{N}_{t} \coloneqq (N_{1,t}, \ldots, N_{k,t})^\top$ is the vector containing the number of pulls of each arm up to round $t$. In GTBs, rewards are sampled from probability distributions whose average rewards vary with the number of triggers, \ie $n \leftarrow \widetilde{N}_{i,t}$ and, consequently, the expected reward of an arm $i$ evolves as $\mu_{i}(\widetilde{N}_{i,t})$. 
Furthermore, we define $t_{i,n} \coloneqq \sum_{l\in [T]} \mathbbm{1} \{ N_{i,l} \le n \}$ as the round in which arm $i$ has been pulled for the $n$-th time. With $\boldsymbol{t}_{i,t} \coloneqq ( t_{i,n} )_{n \le N_{i,t}}$ we refer to the vector containing all the rounds in which the arm $i$ has been pulled, up to time $t$. Moreover, we introduce $t_{i,n}^I \coloneqq \widetilde{N}_{i,t_{i,n}}$, namely the \emph{internal time} of the $n$-th pull of arm $i$, which is the number of triggers of arm $i$ at the time of the $n$-th pull. We also define, given the connectivity matrix of a graph $\mathbf{G}$, the notion of $\bar{k}_1 \coloneqq |\{i \in [k] : \text{deg}(i)=1\}| $ as the number of arms having degree of $1$, where $\text{deg}(i) \coloneqq \mathbf{1}_{k}^\top \mathbf{G} \mathbf{e}_i$ is the degree of a node, \ie the number of edges incident to the node. We now observe the relationship between rested and restless bandits and our setting.

\begin{remark}[Inclusion of Rested and Restless Bandits in GTBs]
GTBs include both \emph{rested} and \emph{restless} bandits~\citep{tekin2012online}. These two settings can be recovered by considering $\mathbf{G}=\mathbf{I}_k$ and $\mathbf{G}=\mathbf{1}_{k \times k}$ for rested and restless settings, respectively.\footnote{We denote $\mathbf{I}_k$ the identity matrix of dimension $k$ and $\mathbf{1}_{k\times k}$ the square matrix of dimension $k$ whose  entries are all equal to $1$.} Indeed, a \emph{restless} bandit can be seen as a particular instance of \settingnameshort where all arms are triggered at each round, making them change every round independently from which action has been chosen ($\widetilde{N}_{i,t}=t$, for every $i \in [k]$). Instead, in a \emph{rested} bandit an arm changes its expected reward only when it is directly chosen ($\widetilde{N}_{i,t}=N_{i,t}$, for every $i \in [k]$).\footnote{This can be easily seen by looking at Equation~\eqref{eq:def_ntilde} considering $\mathbf{G}=\mathbf{I}_k$ and observing that the vector $\mathbf{e}_i$ selects the $i$-th element of vector $\mathbf{N}_t$.}
\end{remark}

\paragraph{Block-Diagonal Connectivity Matrix.}
We now discuss a particular case of GTBs that is interesting from both the practical and analytical point of view. Until now, we considered $\mathbf{G} \in \{ 0, 1\}^{k \times k}$ to be a general binary symmetric matrix. However, we now focus on the specific case in which $\mathbf{G}$ is a \emph{block-diagonal} connectivity matrix, \ie a matrix in which the main-diagonal blocks are square matrices of all ones, and all off-diagonal blocks are zero matrices. Formally, let $\mathbb{B}_{\widetilde{k}}\subset \{0,1\}^{k \times k}$ be the set of block-diagonal connectivity matrices with exactly $\widetilde{k} \in [k]$ distinct blocks of $1$s. We call the \emph{GTBs with block-diagonal connectivity matrix} the set of instances where it holds that $\mathbf{G} \in \mathbb{B}_{\widetilde{k}}$, for some $\widetilde{k} \leq k$. We identify with $\mathcal{C}_\mathbf{G}=\{C_{m,\mathbf{G}}\}_{m\in[\widetilde{k}]}$ the partition of $[k]$ corresponding to the diagonal blocks of $\mathbf{G}$. In graph theory, a block-diagonal connectivity matrix $\mathbf{G} \in \mathbb{B}_{\widetilde{k}}$ corresponds to a cluster graph, \ie a graph formed from the disjoint union of complete graphs or \emph{cliques}~\citep{shamir2004cluster}. We call $\mathcal{C}_\mathbf{G}$ the set of cliques and we indicate with $\widetilde{N}_{C_m,t} \coloneqq \sum_{i \in C_m} N_{i,t}$ the number of times an arm belonging to clique $C_m\in \mathcal{C}_\mathbf{G}$ has been pulled, namely the number of triggers of the clique $C_m$.

\subsection{Learning Problem}
\label{sec:problemformulation:learning}
We define $\mathcal{H}_{t}=\{(I_l, X_{I_l,l})\}_{l\in[t]}$ as the \emph{history of interactions} at a given round $t\in [T]$. We define a policy $\pi (t) $ as a function 
$\pi (t): \mathcal{H}_{t-1} \mapsto I_t $ returning the next action given the history up to that round. 
For a given instance $\boldsymbol{\nu}$ of a \settingnameshort, the performance of a policy $\pi$ is measured by means of \emph{expected cumulative reward} throughout $T$ rounds, formally:
\begin{align*}
    J_{\boldsymbol{\nu},\mathbf{G},T}(\pi)\coloneqq \mathbb{E} \left[ \sum_{t \in [T]}\mu_{I_t}(\widetilde{N}_{I_t,t})\right],
\end{align*}
where the expectation is taken over the randomness of both the environment and the policy/algorithm. A policy $\pi$ is \emph{optimal} for instance $\boldsymbol{\nu}$, a connectivity matrix $\mathbf{G}$, and time horizon $T$ if it maximizes the expected cumulative reward, formally:
$$\pi^*_{\boldsymbol{\nu},\mathbf{G},T} \in \argmax_{\pi} J_{\boldsymbol{\nu},\mathbf{G},T}(\pi).$$ We denote by $J_{\boldsymbol{\nu},\mathbf{G},T}^* = J_{\boldsymbol{\nu},\mathbf{G},T}(\pi^*_{\boldsymbol{\nu},\mathbf{G},T})$ the expected cumulative reward attained by the optimal policy.
We can now define the \emph{expected policy regret} as: 
\begin{align*}
R_{\boldsymbol{\nu},\mathbf{G},T}(\pi) = J_{\boldsymbol{\nu},\mathbf{G},T}^* - J_{\boldsymbol{\nu},\mathbf{G},T}(\pi).
\end{align*}
Therefore, our learning problem is to find a policy $\pi$ minimizing the expected policy regret $R_{\boldsymbol{\nu},\mathbf{G},T}(\pi)$. Since the optimal policy depends simultaneously on $\boldsymbol{\nu}$, $\mathbf{G}$, and $T$, from now on, we consider an instance of the \settingnameshort problem the triple $(\boldsymbol{\nu},\mathbf{G},T)$, instead of the reward distributions $\boldsymbol{\nu}$ only. 

\begin{remark}[On the Chosen Notion of Regret]\label{remark:policyregret} 
In GTBs, we consider a notion of \emph{policy} regret~\citep{DekelTA12}. Indeed, in this setting, diverging from the optimal sequence of actions influences not only instantaneous regret but also leads to a suboptimal history, implying future regret even when returning to an optimal policy from there on. This notion of regret, which shares similarities with the one of reinforcement learning, is more challenging to optimize.
\end{remark}

%% file: content/03rising.tex
\section{Rising Graph-Triggered Bandits}
\label{sec:rising}
Among the various types of restless and rested bandits available in the literature, in this section, we focus on \emph{Rising Bandits}~\citep{heidari2016tight,metelli2022stochastic}. We first introduce the assumption of the rising setting and some useful quantities. Then, we discuss the optimality in this setting (Section~\ref{sec:optimality}). Subsequently, we discuss the regret minimization problem for both the deterministic (Section~\ref{sec:det}) and stochastic (Section~\ref{sec:alg_g_blocchi}) scenarios.

Rising bandits are a specific class of MABs in which the expected reward of each arm evolves in a non-decreasing and concave manner. The following assumption formalizes such behavior. 
\begin{ass}[Non-decreasing and Concave Payoffs]\label{ass:rising}
Let $\boldsymbol{\nu}$ be an instance of a rising bandit, then, defining $\gamma_i(n) \! \coloneqq \! \mu_i(n+1) \! - \! \mu_i(n)$ for every $i\in [k]$ and $n \in [T]$, it holds:
\begin{align*}
    &\text{Non-decreasing:}~~~~~ \; \gamma_i(n) \ge 0, \\ 
    &\text{Concave:}~~~~~ \! \gamma_i(n-1) \ge \gamma_i(n).
\end{align*}
\end{ass}
The two parts of this assumption allow us to provide theoretical guarantees in both the restless and rested settings. Such guarantees cannot be provided without the concavity assumption~\citep[see Theorem~4.2 of][]{metelli2022stochastic}. We call \emph{Rising GTBs}, the instances of GTBs in which the expected rewards fulfill Assumption~\ref{ass:rising}.

\paragraph{Instance Characterization.}
Assumption~\ref{ass:rising} ensures sufficient structure on the problem to allow for algorithms with provably strong theoretical guarantees. 
In this scenario, given an instance $\boldsymbol{\nu}$, we define the \emph{total increment} as:
\begin{align*}
\Upsilon_{\boldsymbol{\nu}}(M, q) \coloneqq \sum_{t\in [M-1]} \max_{i\in[k]} \gamma_i(t)^q,
\end{align*}
where $M \in \mathbb{N}$ and $q \in [0,1]$.
This quantity figures in the (instance-dependent) analysis of algorithms and characterizes the difficulty of learning in instance $\boldsymbol{\nu}$.

\subsection{Optimality in Rising GTBs}
\label{sec:optimality}
In this part, we discuss the notion of \emph{optimality} for our learning problem. 
We first characterize the complexity of finding the optimal policy followed by the clairvoyant when both the expected values and the matrix $\mathbf{G}$ are \emph{known}.

\begin{restatable}[Complexity of Finding the Optimal Policy in Rising GTBs]{thr}{theoremHardness}
\label{thr:hardness}
Computing the optimal policy in Rising GTBs with general matrices $\mathbf{G}$ is NP-hard.
\end{restatable}
This theorem follows from a reduction to the NP-hard problem of determining if a large clique in a given graph exists~\citep{karp1972reducibility}. Intuitively, given a graph $(V,E)$, we build an instance in which the cumulative reward is maximum only if the learner plays a sequence of arms that are associated with vertices in a clique.
Theorem~\ref{thr:hardness} implies that the class of problems of Rising GTBs is computationally harder than all restless bandits and rested rising bandits, for which the optimal policy can be computed in polynomial time~\citep{heidari2016tight}. Moreover, the optimal policy does not admit a simple closed-form representation. Thus, in general, the optimal policy cannot be reduced to a greedy one or to a fixed-arm policy. The result highlights how this definition of optimal policy is closer to that of MDPs rather than that of standard bandit settings. 

We now show how, for the special case of Rising GTBs with block-diagonal connectivity matrices, the optimal policy can be efficiently computed and admits a closed-form solution. 
\begin{restatable}[Optimal Policy in Rising GTBs with Block-Diagonal CM]{thr}{theoremBlockOptimal}\label{thr:optimalblocks} 
For any instance $(\boldsymbol{\nu}, \mathbf{G}, T)$ of Rising GTBs with $\mathbf{G} \in \mathbb{B}_{\widetilde{k}}$, the optimal policy $\pi_{\boldsymbol{\nu}, \mathbf{G}, T}^* \in \argmax_{\pi} J_{\boldsymbol{\nu}, \mathbf{G}, T} (\pi)$ is given by:
\begin{equation*}
    \pi_{\boldsymbol{\nu}, \mathbf{G}, T}^*(t) \in \argmax_{j \in C_{\boldsymbol{\nu},\mathbf{G}, T}^*} \mu_j(t), \qquad \forall t \in [T],
\end{equation*}
where $C_{\boldsymbol{\nu},\mathbf{G},T}^*$ is the \quotes{best} cumulative reward clique: $$C_{\boldsymbol{\nu},\mathbf{G},T}^* \in \argmax_{C \in \mathcal{C}_\mathbf{G}} \sum_{t\in[T]} \max_{j \in C} \mu_j(t).$$
\end{restatable}

\begin{proof}
For each clique $C_m \in \mathcal{C}_\mathbf{G}$, we substitute the reward function of every arm $i \in C_m$ with $\mu_i^* (t) = \max_{i\in C_m} \mu_{i}(t)$, for every $ t\in [T]$. Now, since all arms sharing the same clique have the same reward function, our instance is equivalent to a $\widetilde{k}$-armed bandit problem, where $\widetilde{k}$ is the number of cliques. Since arms in different cliques are not connected, this corresponds to a rested bandit problem, and we use Proposition 1 of~\citep{heidari2016tight} to get that the optimal policy would only pull the best action in terms of cumulative reward at the end of the time horizon $T$. To conclude the proof, we remark that playing greedily inside a clique corresponds exactly to playing on the reward function defined above, which dominates the initial problem, and so the maximum cumulative reward is exactly the one attained in the problem with $\widetilde{k}$ arms.
\end{proof}

This result characterizes the optimal policy when the graph linking the actions is composed only of cliques. In particular, the clairvoyant would play a greedy policy but always inside the same predefined subset of arms composing a clique. Naturally, the chosen clique would be the one having the maximum cumulative reward at the end of the trial. We point out how this policy \quotes{combines} the optimal policies from both rising rested bandits (corresponding to always playing the arm with the highest \emph{cumulative} reward), and the optimal policy from rising restless bandits (the \emph{greedy} policy, corresponding to always playing the arm with the highest \emph{instantaneous} reward).

\subsection{Deterministic Rising GTBs}\label{sec:det}
In this part, we propose two novel algorithms to learn in \emph{deterministic} Rising GTBs, \ie all instances of Rising GTBs where $\sigma=0$. More in detail, in Section~\ref{sec:alg:det:blocchi}, we discuss the block-diagonal CM case, while in Section~\ref{sec:algdet:general}, we discuss the general scenario. The deterministic scenario allows for a better understanding of the complex structure of this setting since it \emph{ignores} the statistical learning problem.

We start by introducing a biased estimator which, for every arm $i\in [k]$, propagates its reward function to the current time $t$ by estimating the first derivative using the last two observations:
\begin{equation}
    \bar{\mu}_i(t) \coloneqq \mu (t_{i,N_{i,t-1}}^I) + (t-t_{i,N_{i,t-1}}^I) \frac{\mu(t_{i,N_{i,t-1}}^I)-\mu(t_{i,N_{i,t-1}-1}^I)}{t_{i,N_{i,t-1}}^I-t_{i,N_{i,t-1}-1}^I}.
    \label{eq:estimator_det}
\end{equation}
This estimator relies on the concept of \emph{internal time}. Internal times are particularly useful since they can separate the bias into two components:
\begin{align*}
    t-t_{i,N_{i,t-1}}^I = \underbrace{(t - t_{i,N_{i,t}}^I)}_{\text{\textcolor{vibrantRed}{(\texttt{A})}}} + \underbrace{(t_{i,N_{i,t}}^I-t_{i,N_{i,t-1}}^I)}_{\text{\textcolor{vibrantBlue}{(\texttt{B})}}}.
\end{align*}
As we will see in Section~\ref{sec:alg:det:blocchi}, this decomposition assumes a particular meaning in instances where $\mathbf{G}\in \mathbb{B}_{\widetilde{k}}$, where \textcolor{vibrantRed}{(\texttt{A})} represents the rested component of the bias, since $t_{i,N_{i,t}}^I = \widetilde{N}_{C_m, t_{i,N_{i,t}}}$ making it equivalent to the bias of a rested bandit where cliques are the arms; and \textcolor{vibrantBlue}{(\texttt{B})} represents the restless component of the bias, since from arm $i$ perspective $t_{i,N_{i,t}}^I = \widetilde{N}_{i, t}$ can be interpreted as the current time inside the clique. 

\subsubsection{Algorithm for Deterministic Rising GTBs with Block-Diagonal CMs}
\label{sec:alg:det:blocchi}

We now introduce \texttt{Deterministic Rising Block-Diagonal Upper Bound} (\algnameblock), an optimistic anytime regret minimization algorithm for deterministic Rising GTBs with block-diagonal connectivity matrix, whose pseudocode is provided in Algorithm~\ref{alg:alg_block_det}. The algorithm takes as input the connectivity matrix $\mathbf{G}$ and employs the estimator presented in Equation~\eqref{eq:estimator_det}. 
Then, after having initialized the counters of the number of pulls, it starts the interaction with the environment. At each round $t \in [T]$, it estimates (line~\ref{blocchi_det:line:updatemean}) the $\bar{\mu}_i(t)$ for every $i\in [k]$ as in Equation~\eqref{eq:estimator_det} and plays greedy according to it (line~\ref{blocchi_det:line:play}).\footnote{At the beginning, the algorithm is required to play every arm $2$ times in a round-robin fashion in order to be able to compute $\bar{\mu}_i(t)$.}

\input{algs/pseudocode_det}

The following result provides the regret bound of \algnameblock, highlighting the impact of the graph topology.
\begin{restatable}[\algnameblock Regret in Det.\ Rising GTBs with Block-Diagonal CMs]{thr}{TheoremdeterministicBlock}
\label{thr:det_block_bound}
    \phantom{a} Let $(\boldsymbol{\nu}, \mathbf{G}, T)$ be an instance of Rising GTB, where $\mathbf{G} \in \mathbb{B}_{\widetilde{k}}$ and $\sigma=0$. Then, \emph{\algnameblock} suffers a regret bounded by:
\begin{align*}
R_{\boldsymbol{\nu},\mathbf{G},T}({\text{\emph{\algnameblockshort}}}) \leq \widetilde{\mathcal{O}} \Bigg( \inf_{q \in [0,1]} \Bigg\{ &  \underbrace{T^q\sum_{C_m \in \mathcal{C}} |C_m| \Upsilon_{\boldsymbol{\nu}}\left(\left\lceil \frac{\widetilde{N}_{C_m,T}}{|C_m|}\right\rceil,q\right)}_{\text{\emph{\textcolor{vibrantRed}{(\texttt{A}) Rested Bias Contribution}}}} + \\ & + \underbrace{\sum_{C_m \in \mathcal{C}} |C_m| \widetilde{N}_{C_m,T}^\frac{q}{1+q}\Upsilon_{\boldsymbol{\nu}}\left(\left\lceil \frac{\widetilde{N}_{C_m,T}}{|C_m|}\right\rceil,q\right)^\frac{1}{1+q}}_{\text{\emph{\textcolor{vibrantBlue}{(\texttt{B}) Restless Bias Contribution}}}}\Bigg\}\Bigg).
\end{align*}
\end{restatable}
\begin{proofsketch}
We can bound the instantaneous regret of \algnameblockshort with the linear bias induced by the estimator:
\begin{align*}
    \mu_{i_t^*}(t) \pm \bar{\mu}_{I_t}(t) - \mu_{I_t}(\widetilde{N}_{I_t,t}) &\le \sum_{t=1}^T\min\{1,(t-t_{I_t,N_{I_t,t-1}}^I)\gamma_{I_t}(t_{I_t,N_{I_t,t-1}-1}^I)\}.
\end{align*}
Then, we can separate the rested contribution from the restless one by splitting the bias:
\begin{align*}
    &\le \sum_{t=1}^T \min\{1,(t - t_{I_t, N_{I_t,t}} ^I)\gamma_{I_t}(t_{I_t,N_{I_t,t-1}-1}^I)\} +\\
    &\qquad + \sum_{t=1}\min\{1,(t_{I_t, N_{I_t,t}} ^I-t_{I_t,N_{I_t,t-1}}^I)\gamma_{I_t}(t_{I_t,N_{I_t,t-1}-1}^I)\}  \label{:tbr:002}
\end{align*}
These two terms represent the rested and the restless contribution to the regret, and we can bound them using similar techniques as in~\citep{metelli2022stochastic}.
\end{proofsketch}
In the bound in Theorem \ref{thr:det_block_bound} (as well as in the bounds in the next sections), $q$ represents the trade-off between being worst-case and exploiting \textit{good}/near-stationary instances (\ie the ones in which the growth $\Upsilon_T$ is strongly sublinear in $T$). A nice feature of our bounds is that they hold for every value of $q$, including the most favorable one, in which the trade-off is optimized. As the shape of $\Upsilon_T$ depends on $q$ itself, it is not possible to express the optimal $q$ minimizing the bound in closed form. Thus, we decided to report the bounds in this general form to highlight their generality. We report the result as a function of the number of triggers $\widetilde{N}_{C_m,T}$ of the cliques in order to better discuss the properties of the graph. However, this dependence can be removed by simply observing $\widetilde{N}_{C_m,T} \le T$. This choice allows us to have an interesting discussion on the nature of this result \wrt the graph structure. We observe that we can separate two contributions to the regret: one coming from the rested behavior (part \textcolor{vibrantRed}{(\texttt{A})} of the bound) determined by the need for identifying the best clique, and the other from the restless behavior needed for identifying the best arm inside the clique (part \textcolor{vibrantBlue}{(\texttt{B})} of the bound). If we compare this result to the bounds in Theorems~4.4 and~5.2 of~\citep{metelli2022stochastic}, we can notice how the shapes of the two contributions correspond. We also remark that, in the two corner cases, \ie rested and restless bandits, the regret bound is actually smaller and corresponds exactly to the bounds presented in~\citep{metelli2022stochastic}, even though this is not immediately visible in Theorem~\ref{thr:det_block_bound} because of a mathematical artifact of the proof.\footnote{More details can be found in Remark~\ref{rem:coincidence_rising} (Appendix~\ref{apx:proofs}).} In the bound, the graph topology emerges by means of cliques' sizes, which act as multiplicative constants. The major consequence is that having fewer cliques leads, in general, to a better bound. As intuition suggests, the rested scenario can lead to a worst-case bound in the first component (which is, by the way, the one having the greater order in $T$), and this can be seen by a simple application of Jensen's Inequality, and by noticing that $\Upsilon_{\boldsymbol{\nu}}$ is a concave function:
$$
\sum_{C_m \in \mathcal{C}} |C_m| \Upsilon_{\boldsymbol{\nu}}\left(\left\lceil \frac{\widetilde{N}_{C_m,T}}{|C_m|}\right\rceil,q\right) \le k\Upsilon_{\boldsymbol{\nu}}\left(\left\lceil \frac{T}{k }\right\rceil,q\right).
$$
We remark that in the two corner cases, one of the two contributions vanishes, even though it cannot be directly seen in Theorem~\ref{thr:det_block_bound}. However, since the restless regret has a better order than the rested one, graphs with fewer cliques may lead, in general, to better bounds. Unfortunately, to precisely quantify this property, one would need to know the exact shape of $\Upsilon_{\boldsymbol{\nu}}$ and to solve a difficult optimization problem.

\subsubsection{Algorithm for Deterministic Rising GTBs with General Matrices}
\label{sec:algdet:general}
After having studied the scenario of block-diagonal connectivity matrices, we now consider the case in which $\mathbf{G}$ can be arbitrary. 
Before introducing the algorithm, we need to define the concept of \emph{block sub-matrix}.

\begin{restatable}[Block Sub-matrix]{defi}{blocksubmatrix}
\label{def:block_submat}
    Let $\mathbf{G} \in \{0,1\}^{k \times k}$ be a general matrix, a block-diagonal matrix $\mathbf{G}^L \in \mathbb{B}_{\widetilde{k}}$ is a \emph{sub-matrix} of $\mathbf{G}$ if it satisfies:
    \begin{equation}
    \label{eq:sub-mat}
        \mathbf{G}_{i,j} - \mathbf{G}_{i,j}^L \ge 0, ~~~\forall i,j \in [k].
    \end{equation}
    Moreover, we say that $\bar{\mathbf{G}}^L \in \mathbb{B}_{\widetilde{k}}$ is \emph{maximal} if it also satisfies:
    \begin{equation*}
        \bar{\mathbf{G}}^L \in \argmin_{\mathbf{G}^L \text{\emph{ \ satisfying Eq.~\eqref{eq:sub-mat}}}} \left| \mathcal{C}_{\mathbf{G}^L} \right| .
    \end{equation*}
\end{restatable}
Informally, $\mathbf{G}^L \in \mathbb{B}_{\widetilde{k}}$ is a sub-matrix of $\mathbf{G}$ if its graph can be obtained by only removing $1$s from $\mathbf{G}$. Finally, a maximal sub-matrix has the least number of cliques. Note that such a maximal sub-matrix is, in general, not unique. 

For this algorithm, we need to introduce a novel estimator, based on sub-matrices, whose definition recalls the one of Equation~\eqref{eq:estimator_det}:
\begin{align}
\bar{\mu}_i^L(t)  \coloneqq \mu (t_{i,N_{i,t-1}}^{I,L}) + (t-t_{i,N_{i,t-1}}^{I,L})\frac{\mu(t_{i,N_{i,t-1}}^{I,L})-\mu(t_{i,N_{i,t-1}-1}^{I,L})}{t_{i,N_{i,t-1}}^{I,L}-t_{i,N_{i,t-1}-1}^{I,L}},
\label{eq:estimator_det_gen}
\end{align}
where $t_{i,l}^{I,L} \coloneqq \mathbf{e}_i^\top (\bar{\mathbf{G}}^L)^\top \mathbf{N}_{t_{i,l}}$ is the internal time \wrt a maximal sub-matrix $\bar{\mathbf{G}}^L$ of the actual matrix $\mathbf{G}$. Given this new estimator, we can generalize Algorithm~\ref{alg:alg_block_det} to attain comparable performance even for a general connectivity matrix $\mathbf{G}$. We introduce a generalization of \algnameblock called \texttt{Deterministic Rising General Upper Bound} (\algnamegendet), whose pseudocode is provided in Algorithm~\ref{alg:alg_gen_det}. The algorithm takes as input a generic matrix $\mathbf{G}$ and computes $\bar{\mathbf{G}}^L$. Then, the algorithm interacts with the environment as before and uses the estimator defined in Equation~\eqref{eq:estimator_det_gen}. In other words, \algnamegendet pretends to be interacting with a bandit with a graph defined by $\bar{\mathbf{G}}^L$. The following result characterizes the performance of \algnamegendet.

\input{algs/pseudocode_det_gen}

\begin{restatable}[\algnamegendet Regret in Det.\ Rising GTBs with General Matrices]{thr}{generaldeterministicBound}
\label{thr:regret_det_gen}
Let $(\boldsymbol{\nu}, \mathbf{G}, T)$ be an instance of Rising GTB, where $\mathbf{G} \in \{0,1\}^{k\times k}$ and $\sigma=0$. Then, \emph{\algnamegendet} suffers a regret bounded by:
\begin{align*}
R&{}_{\boldsymbol{\nu},\mathbf{G},T}({\text{\emph{\algnamegendetshort}}}) \\ & \!\!\! \leq \widetilde{\mathcal{O}}\Bigg( \min_{q \in [0,1]} \Bigg\{T^q \!\!\!\!\! \sum_{C_m^L \in \mathcal{C}_{\bar{\mathbf{G}}^L}} \!\!\!\! |C_m^L| \Upsilon_{\boldsymbol{\nu}}\left(\left\lceil \frac{\widetilde{N}_{C_m^L,T}}{|C_m^L|}\right\rceil,q\right) + \!\!\!\!\!\! \sum_{C_m^L \in \mathcal{C}_{\bar{\mathbf{G}}^L}} \!\!\!\!\! |C_m^L| \widetilde{N}_{C_m^L,T}^\frac{q}{1+q}\Upsilon_{\boldsymbol{\nu}}\left(\left\lceil \frac{\widetilde{N}_{C_m^L,T}}{|C_m^L|}\right\rceil,q\right)^{\!\! \frac{1}{1+q}}\Bigg\} \Bigg),
\end{align*}
where $\bar{\mathbf{G}}^L \in \mathbb{B}_{\widetilde{k}}$ is a maximal sub-matrix of $\mathbf{G}$.
\end{restatable}
This result provides a formal justification to the intuition that the performance of Algorithm~\ref{alg:alg_gen_det} can be bounded with the upper bound attained in a less favorable scenario, \ie a block-diagonal instance that is \quotes{closer} to the worst-case instance of a rested bandit. The regret bound of \algnamegendet can be found by applying Theorem~\ref{thr:det_block_bound} using the matrix $\bar{\mathbf{G}}^L$.

\begin{remark}[Computational Complexity]
Note that retrieving the maximal sub-matrix $\bar{\mathbf{G}}^L$ is NP-hard. This can be easily verified by observing that it is equivalent to the Minimum Clique Partition (MCP) problem, notoriously NP-hard~\citep{karp1972reducibility}. In principle, one could avoid exactly computing the maximal sub-matrix, but rather use a polynomial-time heuristic to find a non-trivial sub-matrix (\ie different from the diagonal matrix) and run \emph{\algnamegendet} using that. The bound in Theorem \ref{thr:regret_det_gen} will scale according to the used sub-matrix.
\end{remark}

\subsection{Stochastic Rising GTBs}
\label{sec:alg_g_blocchi}
In this part, we focus on the \emph{stochastic} Rising GTBs scenario. We characterize the performance of \algnamegenshort~\citep{metelli2022stochastic}, designed for rising rested and restless bandits, in the Rising GTBs setting for both the block-diagonal CMs (Section~\ref{sec:rising:stoch:block}) and the general case (Section~\ref{sec:rising:stoch:general}). We show that such an algorithm achieves good performance for a general $\mathbf{G}$. In particular, we develop a new proof strategy for the regret upper bound that makes graph-dependent terms explicit. We aim at obtaining a computationally efficient algorithm enjoying \emph{sublinear regret} guarantees.\footnote{When we say \emph{sublinear} we refer to the explicit dependence on $T$, and we do not account for $\Upsilon_{\boldsymbol{\nu}}$.} Surprisingly, our analysis shows that \algnamegenshort not only enjoys sublinear regret for any connectivity matrix $\mathbf{G}$, but also that the graph-dependent quantities actually interpolate the regret between the two corner cases. Moreover, we show that there is no need to solve any additional NP-hard problem before or during the algorithm's executions, letting \algnamegenshort keep it affordable computational costs, as in the two corner settings. Furthermore, in this case, the algorithm is completely \emph{unaware} of the graph structure.

The algorithm employs a biased estimator which, for every arm $i$, propagates its reward function to the current round $t$ by estimating the first derivative over the last $2h$ samples:
\begin{align}
\label{eq:estimator_rising_stoc}
     \widehat{\mu}_i^h(t) \! \coloneqq \! \frac{1}{h} \! \sum_{l=N_{i,t-1}-h+1}^{N_{i,t-1}} \!\!  \left( \! { X_{i,t_{i,l}} \! + \! (t-l) \frac{ X_{i,t_{i,l}} -  X_{i,t_{i,l-h}}}{h} } \! \right) \!,
\end{align}
where $h \in \mathbb{N}$ is the window size. We report the estimator's concentration rate, which is a function of the window size $h$. The proof of this result originally appeared in~\citep{metelli2022stochastic}. However, it can be extended to Rising GTBs (more details are provided in Appendix~\ref{apx:lemmas}).

\begin{restatable}[Concentration of Estimator, adapted from \citealt{metelli2022stochastic}]{lemma}{lemmaConcentration}
\label{lemma:lemmaConcentration}
For every arm $i \in [k]$, every round $t \in [T]$, and window width $1 \le h \le \left\lfloor \frac{N_{i,t-1}}{2}\right\rfloor$, let:
\begin{align*}
	\beta^{h}_i(t,\delta)\coloneqq \sigma  (t-N_{i,t-1}+h-1) \sqrt{ \frac{ 10  \log \frac{1}{\delta} }{h^3} }.
\end{align*}
Then, if the window size depends on the number of pulls only $h_{i,t} = h(N_{i,t-1}) $ and if $\delta_t = t^{-\alpha}$ for some $\alpha > 2$, it holds for every round  $t \in [T]$ that:
\begin{equation*}
\mathbb{P}\left(\left| \widehat{\mu}_i^{h_{i,t}}(t) - \widetilde{\mu}_i^{h_{i,t}}(t) \right| >  \beta^{h_{i,t}}_i(t,\delta_t) \right) \le 2t^{1-\alpha}.
\end{equation*}
\end{restatable}

\input{algs/pseudocode_gen}

\paragraph{Algorithm.} The algorithm, whose pseudocode is reported in Algorithm~\ref{alg:alg_gen}, takes as input the subgaussianity proxy $\sigma$, the sliding window size parameter $\epsilon$, and a sequence of properly selected confidence levels $\delta_t$, where $t \in [T]$.
\algnamegenshort relies on the previously defined biased estimator and uses its confidence interval to make decisions in an optimistic manner. \algnamegenshort does not require the time horizon $T$ as an input, making it an anytime algorithm. Moreover, the algorithm exploits the sliding window mechanism to deal with the environment's uncertainty while controlling the confidence degree by means of $\{\delta_t\}_{t \in [T]}$. In particular, the window size employed by the algorithm is proportional to parameter $\epsilon \in (0,1/2)$, in the form of $h_{i,t} = \lfloor \epsilon N_{i,t-1} \rfloor$. As we show below, $\epsilon$ controls the bias-variance trade-off, where low values for $\epsilon$ result in less bias but higher variance, and vice versa.

\begin{remark}[Computational Complexity]
At each round, \emph{\algnamegenshort} only needs to update the estimator and the related confidence bounds for every arm, which can be done in a time linear in the number of arms at every step. For an efficient update, we refer the reader to~\citep[][Appendix~C]{mussi2023best}.
\end{remark}

\subsubsection{Regret for Stochastic Rising GTBs with Block-Diagonal CMs}
\label{sec:rising:stoch:block}

We now analyze the performance of \algnamegenshort in the block-diagonal CMs case.

\begin{restatable}[\algnamegenshort Regret in Rising GTBs with Block-Diagonal CMs] {thr}{regretUBblock}\label{thr:regret_block}
\phantom{a} Let $(\boldsymbol{\nu}, \mathbf{G}, T)$ be an instance of Rising GTB, where $\mathbf{G} \in \mathbb{B}_{\widetilde{k}}$. Let $h_{i,t} = \lfloor \epsilon N_{i,t-1} \rfloor$ for $\epsilon \in (0, 1/2)$ and $\delta_t = t^{-\alpha}$ for $\alpha > 2$. Then, \emph{\algnamegenshort} suffers an expected regret bounded by:
\begin{align*}
& R_{\boldsymbol{\nu},\mathbf{G}, T}(\text{\emph{\algnamegenshort}}) \\
 & \ \leq \widetilde{\mathcal{O}} \Bigg( \min_{q \in [0,1]} \bigg\{\!\!\!\!\! \underbrace{\textcolor{white}{\bigg|}(\vphantom{T^{\frac{2q}{1+q}} \!\!\!\! \sum_{C_m \in \mathcal{C}_{\mathbf{G}} : |C_m| > 1} \!\!\!\!\!\!\!\!\!\!\!\! |C_m| \Upsilon_{\bm{\nu}} \! \left( \left\lceil \frac{T}{|C_m|} \right\rceil \! , q \right)^{\!\! \frac{1}{1+q}}}\sigma T)^{\frac{2}{3}}\textcolor{white}{\bigg|}}_{\substack{\text{\emph{\textcolor{vibrantGrey}{(\texttt{A}) Variance}}} \\ \text{\emph{\color{vibrantGrey} Contribution}}}} \!\!\!\! ~+~ \underbrace{\vphantom{T^{\frac{2q}{1+q}} \!\!\!\! \sum_{C_m \in \mathcal{C}_{\mathbf{G}} : |C_m| > 1} \!\!\!\!\!\!\!\!\!\!\!\! |C_m| \Upsilon_{\bm{\nu}} \! \left( \left\lceil \frac{T}{|C_m|} \right\rceil \! , q \right)^{\!\! \frac{1}{1+q}}} \bar{k}_1 T^q \Upsilon_{\boldsymbol{\nu}} \! \left(\left\lceil \frac{T}{\bar{k}_1}\right\rceil, q\right)}_{\substack{\text{\emph{\textcolor{vibrantRed}{(\texttt{B}) Rested Bias}}} \\ \text{\emph{\color{vibrantRed} Contribution}}}} ~+~  \underbrace{ T^{\frac{2q}{1+q}} \!\!\!\!\!\! \sum_{C_m \in \mathcal{C}_{\mathbf{G}} : |C_m| > 1} \!\!\!\!\!\!\!\!\!\!\!\! |C_m| \Upsilon_{\bm{\nu}} \! \left( \left\lceil \frac{T}{|C_m|} \right\rceil \! , q \right)^{\!\! \frac{1}{1+q}}}_{\substack{\text{\emph{\textcolor{vibrantBlue}{(\texttt{C}) Restless Bias}}} \\ \text{\emph{\color{vibrantBlue} Contribution}}}} \bigg\} \!\! \Bigg) \! ,
\end{align*}
where $\bar{k}_1$ is the number of cliques in $\mathbf{G}$ containing only one action.
\end{restatable}

\paragraph{Existence of a Bias-Variance Trade-off.}
In the regret upper bound, we can observe three distinct contributions. First, \textcolor{vibrantGrey}{(\texttt{A})} represents the variance contribution, which is the regret suffered by the algorithm due to the stochastic nature of the environment. This contribution is due to the estimator's concentration properties and sets a minimum order of regret to $\widetilde{\mathcal{O}}((\sigma T)^{2/3})$. This term is independent of the total increment $\Upsilon_{\boldsymbol{\nu}}$ but, differently from the others, is the only contribution depending on $\sigma$. 
The contribution due to the estimator's bias is split into two distinct parts. The term \textcolor{vibrantRed}{(\texttt{B})}  represents the rested contribution, which scales with the number of blocks containing only one arm. The term \textcolor{vibrantBlue}{(\texttt{C})}, instead, represents the restless contribution that scales with the number and the sizes of cliques. The bias contributions depend explicitly on the shape of average reward functions by total increment $\Upsilon_{\boldsymbol{\nu}}$. The only term common to variance and bias contributions is $\epsilon$. Indeed, $\epsilon$ regulates such a trade-off between bias and variance, and this effect can be observed in the complete expression of the regret upper bound in Appendix~\ref{apx:proofs}. The variance contribution depends linearly on $\epsilon^{-1}$; thus, a smaller window size implies a higher variance in the estimate. On the contrary, the bias tends to increase with $\epsilon$: this is expected since a larger window means including older samples in the estimate.

\paragraph{Dependence on Graph Topology.}
In the regret upper bound of Theorem~\ref{thr:regret_block}, the only contributions depending on graph topology are the bias ones (terms \textcolor{vibrantRed}{(\texttt{B})} and \textcolor{vibrantBlue}{(\texttt{C})}). Indeed, the environment's randomness contribution has been decoupled from the estimation bias to get a tractable stochastic structure. 
We observe how the different behaviors of arms not connected with the others (size-$1$ cliques, corresponding to rested arms) and arms belonging to larger cliques. The regret scales as $T^q$ in rested arms, but the dependence on the total increment $\Upsilon_{\boldsymbol{\nu}}$ is linear. Instead, for cliques with size greater than $1$, regret scales as $T^{\frac{2q}{1+q}}$, which is greater than in rested contribution, but scales with $\Upsilon_{\boldsymbol{\nu}}$ to the power of $\frac{1}{1+q}$, that is indeed a better dependence. Moreover, each clique contributes differently, based on its size. Overall, the higher the size, the higher the contribution is, since the linear term is dominant \wrt the inverse term inside the total increment $\Upsilon_{\boldsymbol{\nu}}$. Another interesting dependence is the one on $\epsilon^{-1}$ for the restless contribution, which can be observed in the complete form of the bound in Appendix~\ref{apx:proofs}. For connected arms, stochasticity and graph topology produce an interaction. Indeed, if one could design an estimator with strong concentration properties for connected arms, this would simplify the analysis of the restless contribution, eliminating the bad dependence on stochasticity. With such an estimator, we conjecture we could reduce the dependence up to $T^{\frac{q}{1+q}}$, matching the deterministic setting bound.\footnote{The lower bounds for rising rested and restless bandits are still an open problem.}

\paragraph{Comparison with Known Results from Literature.}
Given that rested and restless rising bandits are special instances of Rising GTBs, we now comment on how the presented bound links to existing results when Algorithm~\ref{alg:alg_gen} is run over one of those instances. We start from the rested scenario, \ie when $\mathbf{G}=\mathbf{I}_k$. Then, we would have $\bar{k}_1 = k$ and an empty summation in the restless bias contribution. The bound of Theorem~\ref{thr:regret_block} would thus assume the following form:
\begin{align*}
 R_{\boldsymbol{\nu},\mathbf{I}_k, T}&(\text{\algnamegenshort}) \leq \widetilde{\mathcal{O}} \Bigg( \min_{q \in [0,1]} \Bigg\{
(\sigma T)^{\frac{2}{3}}
+ k T^q \Upsilon_{\boldsymbol{\nu}}\left(\left\lceil \frac{T}{k}\right\rceil \! , q \right) \Bigg\} \Bigg) .
\end{align*}
The only other existing result for the rested rising bandits setting is the one of Theorem~4.4 of~\citep{metelli2022stochastic}, which is matched up to constants by ours. In the restless scenario, \ie when $\mathbf{G}=\mathbf{1}_{k \times k}$, we have a unique clique of size $k$, and $\bar{k}_1 = 0$. Thus, the bound we presented in Theorem~\ref{thr:regret_block} becomes:
\begin{align*}
 R_{\boldsymbol{\nu},\mathbf{1}_{k \times k}, T} & (\text{\algnamegenshort}) \leq  \widetilde{\mathcal{O}} \Bigg( \min_{q \in [0,1]} \Bigg \{
(\sigma T)^{\frac{2}{3}} + kT^{\frac{2q}{1+q}}\Upsilon_{\bm{\nu}} \! \left( \left\lceil \frac{T}{k} \right\rceil \! , q \right)^{\!\! \frac{1}{1+q}} \Bigg \} \! \Bigg) .
\end{align*}
Once again, this result matches (up to constants) the result from Theorem~5.3 of~\citep{metelli2022stochastic}, the current state-of-the-art for the restless rising bandits problem. 
To conclude, we generalize the stochastic rising rested/restless bandit setting, with regret bounds that are tight \wrt the known results for the two corner scenarios.

\subsubsection{Regret for Stochastic Rising GTBs with General Matrices}
\label{sec:rising:stoch:general}
We are now ready to generalize the result of Theorem~\ref{thr:regret_block} to general matrices in $\mathbf{G} \in \{0,1\}^{k\times k}$. Before that, we first need to introduce the notion of \textit{block super-matrix}.
\begin{restatable}[Block Super-matrix]{defi}{}
\label{def:block_supermat}
    Let $\mathbf{G} \in \{0,1\}^{k \times k}$ be a general matrix, a block-diagonal matrix $\mathbf{G}^U \in \mathbb{B}_{\widetilde{k}}$ is a \emph{super-matrix} of $\mathbf{G}$ if it satisfies:
    \begin{equation}
    \label{eq:super-mat}
        \mathbf{G}_{i,j} - \mathbf{G}_{i,j}^U \le 0, ~~~\forall i,j \in [k].
    \end{equation}
    Moreover, we say that $\bar{\mathbf{G}}^U \in \mathbb{B}_{\widetilde{k}}$ is \emph{minimal} if it also satisfies:
    \begin{equation*}
        \bar{\mathbf{G}}^U \in \argmax_{\mathbf{G}^U \text{\emph{ \ satisfying Eq.~\eqref{eq:super-mat}}}} \left| \mathcal{C}_{\mathbf{G}^U} \right| .
    \end{equation*}
\end{restatable}
This concept of minimal super-matrix plays an analogous role as the maximal sub-matrix in Theorem~\ref{thr:regret_det_gen}. We now have all the elements to present the upper bound on the regret for the stochastic Rising GTBs case and general matrices. 

\begin{restatable}[\algnamegenshort Regret in Rising GTBs with General Matrices]{thr}{regretUBgeneric}
\label{thr:regret_gen}
Let $(\boldsymbol{\nu}, \mathbf{G}, T)$ be an instance of Rising GTB, where $\mathbf{G}\in \{0,1\}^{k\times k}$. Let $h_{i,t} = \lfloor \epsilon N_{i,t-1} \rfloor$ for $\epsilon \in (0, 1/2)$ and $\delta_t = t^{-\alpha}$ for $\alpha > 2$. Then, \emph{\algnamegenshort} suffers an expected regret bounded by:
\begin{align*}
R_{\boldsymbol{\nu},\mathbf{G}, T}(\text{\emph{\algnamegenshort}}) \! \leq \! \widetilde{\mathcal{O}}\Bigg( \! \min_{q\in[0,1]} \! \Bigg\{ \! (\sigma T)^\frac{2}{3}\! + \! T^q \bar{k}_1 \Upsilon_{\nu} \! \left(\frac{T}{\bar{k}_1},q\right) \! + \! T^\frac{2q}{1+q}\sum_{C_m^U}|C_m^U|\Upsilon_{\nu} \! \left(\frac{T}{|C_m^U|},q\right)^\frac{1}{1+q} \! \Bigg\} \! \Bigg) \! ,
\end{align*}
where $\bar{\mathbf{G}}^U$ is the minimal super-matrix of $\mathbf{G}$.
\end{restatable}

This result is obtained by bounding $\widetilde{N}_{C_m^U,T} \le T$ for every $C_m^U \in \mathcal{C}_{\bar{\mathbf{G}}^U}$ to remove any stochastic quantity from the regret, but a more precise bound can be provided by finding the worst-case allocation of the triggers among the cliques (as discussed for the similar result in Theorem~\ref{thr:regret_det_gen}). However, this would require solving a challenging optimization problem that does not admit any closed-form solution.
This result is similar to the one presented in Theorem~\ref{thr:regret_det_gen}, with the only difference being that the dependence on graph topology is linked to the minimal super-matrix. In principle, the result holds for any super-matrix of $\mathbf{G}$. Still, in the stochastic setting, the upper bound for the rested scenario is better than the one for the restless scenario. Hence, a block-diagonal CM with as many cliques as possible will, in most cases, lead to better bounds. 

\paragraph{About the Knowledge of $\mathbf{G}$.}
In the stochastic scenario, we avoid extracting the super-matrix structure from the graph before executing the algorithm, as it always plays the same policy, regardless of the graph structure. Indeed, Algorithm~\ref{alg:alg_gen} \emph{does not require the knowledge of the graph}: the algorithm plays as if the true matrix is the identity one (\ie a rested instance). 
To justify this behavior in an intuitive way, we have to look at Theorems~4.4 and~5.3 of~\citep{metelli2022stochastic}: in stochastic scenarios, the \emph{rested} contribution to regret upper bound has a better dependence on $T$ \wrt the restless one. Moreover, our optimistic estimator computed by assuming a less connected graph will always be higher than the one computed from any more densely connected graph. Thus, by playing a purely rested policy, we are always sure to over-estimate the true reward (\ie optimism holds) and we are guaranteed that the rested contribution to the regret is maximized \wrt the restless contribution. The final form of the regret bound is obtained by including the minimal super-matrix as a pessimistic proxy of the effect of connected arms (informally, the minimal super-matrix represents the maximum possible contribution to the regret that is due to the arms' connections). 
We point out that Algorithm~\ref{alg:alg_gen} does not require the minimal super-matrix as an input, as it is needed only in the analysis. For this reason, one could reformulate the following result by removing the dependence on the minimal super-matrix and including a minimization over the set of all super-matrices. As a side effect, this dramatically reduces the computational burden \wrt the deterministic setting at the cost of a slightly worst regret bound. 

\paragraph{Comparison with Deterministic Regret Bounds.}
In the deterministic scenario (Theorems~\ref{thr:det_block_bound} and~\ref{thr:regret_det_gen}), the restless contributions are always of smaller order compared to the rested one, which is the contrary of what we observe in stochastic settings (Theorems~\ref{thr:regret_block} and~\ref{thr:regret_gen}). For this reason, in Algorithm~\ref{alg:alg_gen_det}, the regret bound scales with the maximal sub-matrix instead of the minimal super-matrix. In the deterministic setting, the maximal sub-matrix represents the maximum possible contribution to the regret that is due to the \emph{absence} of arms connections. In principle, we could remove the necessity for graph knowledge also in the deterministic setting by simply playing as in a rested scenario (\ie run Algorithm~\ref{alg:alg_block_det} by setting $\mathbf{G}=\mathbf{I}_{k}$). This would be sensibly suboptimal since any graph connection can be used to obtain a strictly better regret bound. This is not the case for the stochastic setting, where over-estimating the number of connections (\eg by playing as in a restless scenario) may result in a non-optimistic estimator, compromising the theoretical soundness of our algorithms.

%% file: algs/pseudocode_det.tex
\RestyleAlgo{ruled}
\LinesNumbered
\begin{algorithm}[t!]
\caption{\algnameblock.}\label{alg:alg_block_det}
\SetKwInOut{Input}{Input}
\small
\Input{Connectivity matrix $\mathbf{G} \in \mathbb{B}_{\widetilde{k}}$}

\For{$t \in [T]$}{

Compute $\bar{\mu}_{i}(t)$ as in Equation~\eqref{eq:estimator_det} \label{blocchi_det:line:updatemean}, $ \ \forall i \in [k]$

Select $ I_t \in \argmax_{i \in [k]} \bar{\mu}_{i}(t)$ 

Play $I_t$ and observe $\mu_{I_t}(\widetilde{N}_{I_t, t})$ \label{blocchi_det:line:play}

}

\end{algorithm}

%% file: algs/pseudocode_det_gen.tex
\RestyleAlgo{ruled}
\LinesNumbered
\begin{algorithm}[t!]
\caption{\algnamegendet.}\label{alg:alg_gen_det}
\SetKwInOut{Input}{Input}
\small
\Input{Connectivity matrix $\mathbf{G}$}

Compute maximal sub-matrix $\bar{\mathbf{G}}^L$ from $\mathbf{G}$

\For{$t \in [T]$}{

Compute $\bar{\mu}_{i}^L(t)$ as in Equation~\eqref{eq:estimator_det_gen}, $\ \forall i \in [k]$

Select $ I_t \in \argmax_{i \in [k]} \bar{\mu}_{i}^L(t)$

Play $I_t$ and observe $\mu_{I_t}(\widetilde{N}_{I_t, t})$ \label{generale_det:line:play} 

}

\end{algorithm}

%% file: algs/pseudocode_gen.tex
\RestyleAlgo{ruled}
\LinesNumbered
\begin{algorithm}[t!]
\caption{\algnamegen.}\label{alg:alg_gen}
\SetKwInOut{Input}{Input}
\small
\Input{
Subgaussianity proxy $\sigma$, confidence levels $\{\delta_t\}_{t \in [T]}$, window size $\epsilon \in (0, 1/2)$.
}

\For{$t \in [T]$}{
Compute $\widehat{\mu}_{i}^{h_{i,t}}(t)$ as in Equation~\eqref{eq:estimator_rising_stoc}, $ \ \forall i \in [k]$

Select $ I_t \in \argmax_{i \in [k]} \widehat{\mu}_{i}^{h_{i,t}}(t) + \beta_{i}^{h_{i,t}}(t, \delta_t)$ \label{line:ucb}

Play $I_t$ and observe $X_{I_t, t}$ \label{line:play}

}
\end{algorithm}

%% file: content/04rotting.tex
\section{Rotting Graph-Triggered Bandits}
\label{sec:rotting}
\textit{Rotting bandits}~\citep{levine2017rotting} are an important family of evolving reward bandits where, contrary to what happens in rising bandits, the reward functions are not allowed to grow. In this section, we explore how the graph-triggering mechanism interacts with the non-increasing reward function assumption. We characterize the optimal policies and the challenges in finding them (Section~\ref{sec:rotting:opt}). Then, we study the regret minimization problem for this setting in the presence of stochastic noise (Section~\ref{sec:rotting:st}).\footnote{For Rotting GTBs, we skip the deterministic case, as all the interesting results we want to show are visible also in the presence of noise.} Before that, we start by stating the main setting assumption and presenting the quantities characterizing this specific kind of bandits. 

\begin{ass}[Non-increasing Payoffs]\label{ass:rotting}
Let $\boldsymbol{\nu}$ be an instance of a rotting bandit, then, defining $\gamma_i(n) \! \coloneqq \! \mu_i(n+1) \! - \! \mu_i(n)$ for every $i\in [k]$ and $n \in [T]$, it holds:
\begin{align*}
    &\text{Non-increasing:}~~~~~ \; \gamma_i(n) \le 0. 
\end{align*}
\end{ass}
This assumption allows for strong theoretical guarantees in both the restless and rested settings, as it has been shown in the literature~\citep[see, \eg ][]{heidari2016tight,levine2017rotting,seznec2019rotting,seznec2020single}. Notably, for rotting bandits, we are not required to have a concavity/convexity assumption.

\paragraph{Instance Characterization.} 
In this scenario, given an instance $\boldsymbol{\nu}$, we define the \emph{total decrement} as:
\begin{align*}
V_{\boldsymbol{\nu}}(M) \coloneqq \sum_{n \in [M-1]} \max_{i \in [k]} |\gamma_i(n)|,
\end{align*}
where $M \in \mathbb{N}$.
Moreover, we define the \emph{maximum per-round variation} as:
\begin{align*}
L \coloneqq \max_{i \in [k]} \max_{n \in [T]} |\gamma_i(n)|,
\end{align*}
with $\mu_i(-1) \coloneqq \max_{i \in [k]} \mu_i(0)$.
These quantities figure in the instance-dependent guarantees of algorithms operating in this setting and characterize the difficulty of learning for the instance $\boldsymbol{\nu}$. In particular, $V_{\boldsymbol{\nu}}(T)$ is required to properly bound the regret in restless rotting bandits (see, \eg \citealt{seznec2020single}), while $L$ appears in the minimax regret bound of rested rotting bandits, as shown in the setting's lower bound of $\Omega (kL)$ of~\citep{heidari2016tight}.

\subsection{Optimality in Rotting GTBs}
\label{sec:rotting:opt}
Under the standard literature's assumptions, we are now ready to characterize our optimal policies. We first show, as for Rising GTBs, a negative result on the complexity of finding the optimal policy for a clairvoyant who knows all about our Rotting GTBs instance. 
\begin{restatable}[Complexity of finding the Optimal Policy in Rotting GTBs]{thr}{theoremHardnessRotting}
\label{thr:hardnessRotting}
    Computing the optimal policy in Rotting GTBs with general matrices $\mathbf{G}$ is NP-hard.
\end{restatable} 
The proof of this result follows a similar logic to the one of Theorem~\ref{thr:hardness}. Given this result, we now proceed by studying the block-diagonal connectivity matrices scenario, which composes an interesting class of Rotting GTBs. We now characterize the optimal policy in the block-diagonal connectivity scenario and the total cumulative reward that it obtains. 

\begin{restatable}[Optimal Policy in Rotting GTBs with Block-Diagonal CM]{thr}{optimalityRotting}
\label{thr:opt}
For any instance $(\boldsymbol{\nu}, \mathbf{G}, T)$ of Rotting GTBs such that $\mathbf{G} \in \mathbb{B}_{\widetilde{k}}$, the optimal policy $\pi_{\boldsymbol{\nu}, \mathbf{G}, T}^* \in \argmax_{\pi} J_{\boldsymbol{\nu}, \mathbf{G}, T} (\pi)$ is given by:
\begin{equation*}
    \pi_{\boldsymbol{\nu}, \mathbf{G}, T}^*(t) \in \argmax_{j \in [k]} \mu_j(\widetilde{N}_{j,t}^*), \qquad \forall t \in [T],
\end{equation*}
where $\widetilde{N}_{j,t}^*$ is the number of times arm $j$ has been triggered by the optimal policy up to time $t$. Moreover, we have:
\begin{equation}
    \label{eq:rotting_opt}
    J_{\boldsymbol{\nu}, \mathbf{G}, T}^* = \sum_{C_m \in \mathcal{C}_\mathbf{G}} \sum_{n=1}^{N_{C_m,T}^*} \max_{i \in C_m} \mu_i(n),
\end{equation}
where $N_{C_m,T}^*$ is the number of times the optimal policy pulls an action belonging to clique $C_m$ before $T$, \ie $N_{C_m,T}^* = \widetilde{N}_{i,T}^*$, for every $i \in C_m$.
\end{restatable}
Interestingly, the optimal policy is \textit{greedy} in every rotting bandit with block-diagonal CM: this extends known results for rested and restless rotting bandits, where the optimal policy was already proven to be greedy~\citep{heidari2016tight,levine2017rotting}. Equation~\eqref{eq:rotting_opt} provides a closed form of the total reward obtained by the optimal policy, which will be useful in the next part.

\subsection{Stochastic Rotting GTBs}
\label{sec:rotting:st}

In this part, we discuss the regret minimization problem for the stochastic Rotting GTBs. We first study an algorithm, namely \rawucb~\citep{seznec2020single}, which is able to achieve sublinear regret in the block-diagonal connectivity scenario (Section~\ref{sec:rotting:st:block}). Then, we show that, under the literature's standard assumptions, we cannot learn for general matrices (Section~\ref{sec:rotting:st:general}).

\subsubsection{Algorithm for Stochastic Rotting GTBs with Block-Diagonal CMs}
\label{sec:rotting:st:block}

We now show that the \rawucb algorithm~\citep{seznec2020single}, whose pseudocode is provided in Algorithm~\ref{alg:rawucb}, provides sublinear regret guarantees in the Rotting GTB setting with block-diagonal connectivity matrices. \rawucb does not require any knowledge on $\mathbf{G}$ and allows for efficient computation.\footnote{More details on the computationally efficient version of \rawucb, namely \texttt{EFF-RAW-UCB}, can be found in~\citep{seznec2020single}.}

\input{algs/pseudocode_rawucb}

The behavior of \rawucb is characterized as follows. At each round $t \in [T]$, the algorithm computes a family of estimators for every action (line~\ref{line:est_rawucb}). In particular, for every action $i \in [k]$ and for every window size $h_{i,t} \in [N_{i,t-1}^\pi]$, it computes:
\begin{equation}
    \label{eq:estimator_rotting}
    \widehat{\mu}_i^h(t) \coloneqq \frac{1}{h}\sum_{s=1}^{t-1} \mathbbm{1}_{\{I_t = i ~\land~ N_{i,s}>N_{i,t-1}-h\}} X_{i,s}.
\end{equation}
Then, for every action, the chosen window size is the one minimizing the upper confidence bound $\widehat{\mu}_{i}^h(t) + c(h, \delta_t)$ where $c(h, \delta_t) \coloneqq \sqrt{2\sigma^2 \log(2\delta_t^{-1})/h}$ (line~\ref{line:ucb_rawucb}).\footnote{At the beginning of the execution, we need a round-robin pull of the arms to initialize the estimators.} Proving the algorithm's guarantees in the rested and restless setting requires characterizing how it concentrates, as has been done for the base case in Lemma~2 of~\citep{seznec2020single}. We extend this result to the Rotting GTBs setting, devising a concentration bound involving the number of triggers of an action. The result can be found in Lemma~\ref{lemma:ucb_gt} (Appendix~\ref{apx:rotting}). This result will play a key role in the regret analysis of \rawucb in the Rotting GTBs setting. We are now ready to state the regret upper bound of \rawucb in the Rotting GTBs for block-diagonal connectivity matrices.

\begin{restatable}[\texttt{RAW-UCB} Regret in Rotting GTBs with Block-Diagonal CMs] {thr}{regretUBblockROT} \label{thr:regret_block_ROT}
Let $(\boldsymbol{\nu}, \mathbf{G}, T)$ be an instance of the Rotting GTBs, where $\mathbf{G} \in \mathbb{B}_{\widetilde{k}}$. Let $\delta_t = t^{-\alpha}$ for $\alpha \ge 5$. Then, \emph{\texttt{RAW-UCB}} suffers an expected regret bounded as:
\begin{align*}
    R_{\boldsymbol{\nu},\mathbf{G},T}(\emph{\texttt{RAW-UCB}}) \le & \  \widetilde{\mathcal{O}} \Bigg( \underbrace{\vphantom{\sum_{C_m \in \mathcal{C}_\mathbf{G}}\left(\sqrt{\frac{|C_m|}{k}T \log T}\right)}k\left(\sigma \sqrt{\log T} + V_{\boldsymbol{\nu}}(T)\right)}_{\text{\emph{\textcolor{vibrantGrey}{(\texttt{A}) Variance Contribution}}}} + \underbrace{ L \!\! \sum_{C_m \in \mathcal{C}_\mathbf{G}} \!\!\! |C_m|^2 \! + kL + \sigma \!\!\!\! \sum_{C_m \in \mathcal{C}_\mathbf{G}} \!\!\! \left(\sqrt{\frac{|C_m|}{k}T}\right)}_{\text{\emph{\textcolor{vibrantRed}{(\texttt{B}) Rested Contribution}}}} \\ 
    & \ \ \ \ + \underbrace{(\alpha \sigma)^\frac{2}{3}\sum_{C_m \in \mathcal{C}_{\mathbf{G}} }\left(V_{\boldsymbol{\nu}}(T)\frac{|C_m|}{k}T^2\right)^\frac{1}{3}}_{\text{\emph{\textcolor{vibrantBlue}{(\texttt{C}) Restless Contribution}}}}\Bigg).
\end{align*}
\end{restatable}

\paragraph{Dependence on Graph Topology.} In Theorem~\ref{thr:regret_block_ROT}, we observe the same phenomenon of Theorem~\ref{thr:regret_block}. Indeed, we have three components: \textcolor{vibrantGrey}{(\texttt{A})} representing the fixed regret contribution that comes from the noise, \textcolor{vibrantRed}{(\texttt{B})} representing the contribution to the regret coming from the rested nature of the problem (\ie the suboptimality accrued by choosing an action in the wrong clique), and \textcolor{vibrantBlue}{(\texttt{C})} representing the contribution coming from the restless nature of the problem, instead (\ie the suboptimality accrued by not choosing the best action in a clique). The separation between the latter two components becomes clear in the proof of the result:
\begin{align*}
        R_{\boldsymbol{\nu},\mathbf{G},T}(\texttt{RAW-UCB}) &= \sum_{t=1}^T \mu_{i_t^*}(\widetilde{N}_{i_t^*, t}^*) - \mu_{I_t}(\widetilde{N}_{I_t, t}^\pi) \pm \max_{i \in C_{I_t}} \mu_{i}(\widetilde{N}_{I_t, t}^\pi) \\
        &=\underbrace{\sum_{t=1}^T (\mu_{i_t^*}(\widetilde{N}_{i_t^*, t}^*) - \max_{i \in C_{I_t}} \mu_{i}(\widetilde{N}_{I_t, t}^\pi))}_{\le \, \text{\color{vibrantRed}(\texttt{B})} \, + \, {\color{black}k\sigma\sqrt{\log T}}} + \underbrace{\sum_{t=1}^T (\max_{i \in C_{I_t}} \mu_{i}(\widetilde{N}_{I_t, t}^\pi)-\mu_{I_t}(\widetilde{N}_{I_t, t}^\pi))}_{\le \, \text{\color{vibrantBlue}(\texttt{C})} \, + \, {\color{black} kV_{\boldsymbol{\nu}}(T)}}.
    \end{align*}
In rotting bandits, there is a clear hierarchy between the difficulties of statistical learning in the rested and the restless settings. Rested rotting bandits are easier than their restless counterpart, and this is reflected also in our bound: when the number of cliques is higher, and the GTB is closer to a rested instance, the regret bound is lower since the weight of \textcolor{vibrantRed}{(\texttt{B})} is higher.

\paragraph{Comparison with Known Results from Literature.} 
\rawucb has already been proven to be nearly optimal in both the rested and the restless scenario~\citep{seznec2020single}. However, as an artifact of the analysis, we cannot retrieve the exact same bounds by plugging $\mathbf{G}= \mathbf{I}_k$, or $\mathbf{G}= \mathbf{1}_{k\times k}$ in our expression. A similar consideration to the one of Remark~\ref{rem:coincidence_rising} can also be done for rotting bandits. We conjecture that \rawucb is actually nearly-optimal also in the intermediate Rotting GTB instances with block-diagonal connectivity matrices, and this claim is supported by the fact that there is no room for improvement in either of the two contributions in the corner cases. We also conjecture that the dependence on $L\sum_{C_m \in \mathcal{C}_\mathbf{G}} |C_m|^2$ may be an artifact of the analysis, being the output of a delicate pigeonhole principle argument used to prove Theorem~\ref{thr:regret_block_ROT} (see Appendix~\ref{apx:rotting}). We leave the task of finding a graph-dependent lower bound for this setting as a fascinating open problem.

\subsubsection{Non-learnability for Stochastic Rotting GTBs with General Matrices}
\label{sec:rotting:st:general}
We now move to the scenario of stochastic Rotting GTBs for general connectivity matrices, and we present an impossibility result for this case. 
\begin{restatable}[Regret Lower Bound for Rotting GTBs with General Matrices]{thr}{rottingregretLBgeneric}
\label{thr:rottingregretLBgeneric}
\phantom{a} For every $\mathbf{G} \in \{0,1\}^{k \times k}$ that is not block-diagonal, there exists an instance of Rotting GTB $(\boldsymbol{\nu}, \mathbf{G}, T)$ such that, for every policy $\pi$, it holds: 
\begin{align*}
R_{\boldsymbol{\nu}, \mathbf{G}, T}(\pi) \geq \frac{T}{12}.
\end{align*}
\end{restatable}

\begin{proofsketch}
We build two deterministic instances that look exactly the same for the first half of the game, but require opposite decisions in the second half.
Take the $3$-arm problem, where pulling arm $1$ also triggers arms $2$ and $3$, while arms $2$ and $3$ do not trigger each other. Now define two instances. In both instances, arms $2$ and $3$ give the reward $2/3$ for their first $T/2$ triggers, and then $0$; in instance $\boldsymbol{\nu}$, arm $1$ gives maximum reward for its first $T/2$ triggers, and then $0$, while for instance $\boldsymbol{\nu}'$ it gives reward $1$ for all $T$ rounds. The two instances are indistinguishable until time $T/2$. Therefore, policies behave in the same way on $\boldsymbol{\nu}$ and $\boldsymbol{\nu}'$ during the first half. If the policy pulls arm $1$ often in the first half, then in instance $\boldsymbol{\nu}$ it wastes the valuable future triggers of arms $2$ and $3$, because pulling arm $1$ also consumes them. So $\boldsymbol{\nu}$ incurs large regret. If instead the policy often pulls arms $2$ and $3$ in the first half, then in instance $\boldsymbol{\nu}'$ it fails to exploit the fact that arm $1$ remains optimal forever, and again incurs large regret. The general case follows from Lemma~\ref{lemma:graph_inclusion}, where we prove that every non-block-diagonal matrix can contain this structure.
\end{proofsketch}

This negative result poses several limitations to what can be obtained from Rotting GTBs in the general scenario. The difficulty of these instances lies in the graph structure, rather than in the statistical challenge of learning the reward functions: indeed, the lower bound already holds for deterministic instances, where there is no observation noise and no estimation error. The obstruction is structural. Pulling one action may irreversibly consume the future rewards of other actions through the triggering graph, and different instances that are indistinguishable in the early rounds may require incompatible triggering schedules to be near-optimal later on.
We emphasize that the condition in Theorem~\ref{thr:rottingregretLBgeneric} concerns the whole connectivity matrix. Hence, the lower bound is not avoided by the presence of a block-diagonal sub-matrix. If $\mathbf{G}$ cannot be permuted into a block-diagonal matrix whose blocks are complete cliques, then Lemma~\ref{lemma:graph_inclusion} identifies a local non-block-diagonal obstruction. The hard three-arm construction can be embedded on this obstruction, while the remaining arms can be assigned dominated rewards and made irrelevant. Thus, any block-diagonal part of the graph does not remove the linear-regret lower bound. Consequently, without additional assumptions on the reward functions, no algorithm can obtain sublinear regret for the class of general connectivity matrices.

%% file: algs/pseudocode_rawucb.tex
\RestyleAlgo{ruled}
\LinesNumbered
\begin{algorithm}[t!]
\caption{\texttt{RAW-UCB}~\citep{seznec2020single}}
\label{alg:rawucb}
\SetKwInOut{Input}{Input}
\small
\Input{Subgaussianity proxy $\sigma$, confidence levels $\{\delta_t\}_{t \in [T]}$.}

\For{$t \in [T]$}{

Compute $\widehat{\mu}_{i}^h(t)$ as in Equation \eqref{eq:estimator_rotting}, $\ \ \forall i \in [k], \ h \in [N_{i,t-  1}]$ \label{line:est_rawucb}

Select $ I_t \in \argmax_{i \in [k]} \min_{h \le N_{i,t-1}} \widehat{\mu}_{i}^h(t) + c(h, \delta_t)$ \label{line:ucb_rawucb}

Play $I_t$ and observe $X_{I_t,t}$

}
\end{algorithm}

%% file: content/05conclusions.tex
\section{Discussion and Conclusions}
\label{sec:conclusions}

In this paper, we proposed \emph{graph-triggered bandits} (GTBs), a generalization of the rested and restless bandit settings, where the expected rewards of the different arms evolve by means of a graph. We focused on and compared two special families of bandits, namely \textit{rising} and \textit{rotting} bandits, where the expected rewards of an arm evolve in a monotonic fashion with the number of \textit{triggers} the specific arm received. We showed that computing the optimal policy without additional assumptions on the matrix is NP-hard in both the rising and rotting scenarios. Then, for both these classes, we showed how, instead, for block-diagonal connectivity matrices, we can find the optimal policy in polynomial time and have a convenient closed-form expression. From the algorithmic perspective, we showed how it is possible to achieve sublinear regret for both of these special instances of MABs with block-diagonal connectivity matrices. On the other hand, for general matrices, we have an interesting distinction. Indeed, for Rising GTBs we were able to achieve sublinear regret, while for Rotting GTBs (in the standard scenario, without additional assumptions on the behavior of expected rewards, \eg on second-order derivatives), we proved that we cannot learn. 

\paragraph{Future Work.} 
This work aspires to be a first step in the study of GTBs and should be extended by studying the statistical complexity of learning through lower bounds, and by considering more general models, \eg smoothly evolving bandits. Moreover, a natural extension of the GTB framework consists in replacing the binary connectivity matrix $\mathbf{G}\in\{0,1\}^{k\times k}$ with a weighted matrix $\mathbf{G}\in[0,1]^{k\times k}$. In such a setting, the entry $G_{i,j}$ would encode the magnitude by which pulling arm $i$ triggers the evolution of arm $j$, rather than only indicating whether such an interaction is present or absent. This would increase the expressiveness of the framework by allowing partial and heterogeneous interactions among arms. For example, pulling an arm could strongly affect some neighboring arms while only mildly affecting others. The binary model studied in this work corresponds to the special case in which all interaction weights belong to \(\{0,1\}\). Extending the learning algorithms and regret analysis to this weighted setting is an interesting direction for future work.

%% file: content/06apx_related.tex
\section{Related Works}
\label{sec:related}
In this appendix, we discuss the relevant literature for the \settingnameshort setting. We divide this appendix into three parts. First, we discuss the relevant works concerning graph structures. Then, we discuss the literature related to restless and rested bandits, with particular attention to rotting and rising bandits. Finally, we discuss related works on competitive analysis.

\paragraph{Graph Relationships in Bandits.}
The graph-triggered bandits setting has been introduced in this work. Thus, no prior literature is available on this setting. However, we mention similar settings that appeared in the last years. \citet{herlihy2023networked} propose the networked restless bandit setting. Despite some similarities with our setting, \eg the presence of a graph among arms, their action space and learning objectives radically differ from ours, and thus the two settings are not comparable. In~\citep{jhunjhunwala2018class}, a restless bandit setting is proposed in which the graph structure is not explicit in the formulation; however, the authors develop a graph representation of the policies in the deterministic scenario. Their algorithm builds and exploits a graph in an online fashion. Once again, we cannot properly compare this setting to ours, despite some sparse similarities. Finally, we mention bandits with graph feedback~\citep{alon2015online}. Despite this setting being conceptually different from ours since arms do not interact, we report it here just because it features graph topology-dependent bounds. We remark that in this case, the graph should not be intended as a structure for arms interactions but rather as a feedback structure for the learner; in GTBs, the feedback is purely bandit.

\paragraph{Rested and Restless Bandits.}
Restless and rested bandits are a well-established research field. Starting from the seminal paper by \citet{whittle1988restless} on restless bandits, several approaches have been proposed over the years to deal with non-stationary bandits~\citep{tekin2012online,raj2017taming}. Then, specialization of these settings such as \textit{rising}~\citep{metelli2022stochastic,mussi2023best} and \textit{rotting}~\citep{levine2017rotting} have been introduced. Over the last years, several works tackled rotting bandits~\citep{levine2017rotting,seznec2019rotting}. Remarkably, \citep{seznec2020single} provide a single algorithm for dealing with both rested and restless rotting bandits but show that in the rotting setting, achieving sublinear regret is not possible when there are both rested and restless arms in the same instance. We remark that for any two-armed rotting bandit where one arm is rested and the other is restless, we can construct an (asymmetric) matrix $\mathbf{G}$ such that the instance can be mapped to a graph-triggered rotting bandit instance. This highlights a crucial difference between rotting and rising bandits for what concerns graph-triggering. Recently, literature studied a broader class of restless bandits called \emph{smooth} bandits, which generalizes both rotting and rising bandits \citep{manegueu2021generalized, jia2023smooth}.

\paragraph{Rising Bandits and Competitive Analysis.}
Recent works on rested rising bandits shifted the focus from regret minimization to competitive analysis~\citep{borodin2005online}. In competitive analysis, the goal is to bound the ratio between the performance of the optimal policy and that of the algorithm, following a \textit{worst-case} principle. In such settings, stochasticity is usually not involved, and the focus is on deterministic, yet unpredictable, processes. In their deterministic version, rested rising bandits constitute an interesting example: a decision-maker is faced with an unpredictable sequence (the evolution of the reward functions) and has to balance immediate gain with worst-case planning. In \citep{patil2022mitigating}, the authors are the first to provide a lower bound on the competitive ratio of any deterministic algorithm of $\Omega(k)$, and an anytime deterministic algorithm with matching guarantees. In \citep{blum2024nearly}, the authors focus on \emph{randomized} algorithms\footnote{The distinction between deterministic and randomized algorithms is customary in the competitive analysis literature, and is a natural way of characterizing the strength of an adversary.} and show that the lower bound on the competitive ratio is $\Omega(\sqrt{k})$. They provide a randomized algorithm matching this lower bound by knowing the maximum value achieved by the best arm, and a randomized algorithm achieving a competitive ratio in $\mathcal{O}(\sqrt{k} \log k)$ without this knowledge. As a future research direction, it could be interesting to explore the possibility of performing competitive analysis in graph-triggered rising bandits.

%% file: content/07apx_rising.tex
\section{Proofs on Rising Bandits}
\label{apx:proofs}

In this appendix, we report a short version of the proofs of Rising GTBs. The extended version is provided in~\citep{genalti2024graph}.

\theoremHardness*

\begin{proof}
	We reduce from a decision problem related to finding cliques in graphs.
	In particular, given a graph $(V,E)$ and $\widetilde{M} \in \mathbb{N}$, it is NP-hard to determine if there exists a clique of size $\widetilde{M}$~\citep{karp1972reducibility}. 
		In the following, we design an instance of our problem such that the reward of the optimal policy is at least $\sum_{t=1}^T (1+\frac{t}{T^2})$ if and only if there exists a clique of size $\widetilde{M}=T$. 
		
		\textbf{Construction.}
		Given a graph $(V,E)$, we build an instance such that the horizon is $T$.
		Our set of actions can be constructed by assigning an action to every node and time step couple, \ie $\mathcal{A}= \{a_{v,t}\}_{v\in V,~ t\in[T]}$.  We define the matrix $\widetilde{\mathbf{G}}$ is such that for any $v,v' \in V$ and $t,t' \in [T]$, it holds $G_{a_{v,t},a_{v',t'}}=1$ if $(v,v')\in E$, and $G_{a_{v,t},a_{v',t'}}=0$ otherwise.
		Finally, for each arm $a_{v,t} \in \mathcal{A}$, the reward is deterministic and evolves as ${\mu}_{a_{v,t}}(n)= \min\{1+ \eta t, \frac{n}{t}  (1+ \eta t)\}$, where $\eta=T^{-2}$. We call $\widetilde{\boldsymbol{\nu}}$ the set of these functions. 
		It is easy to see that the \settingnameshort instance $(\widetilde{\boldsymbol{\nu}}, \widetilde{\mathbf{G}}, T)$ satisfies \cref{ass:rising}.
		
		\textbf{if.}
		We show that if there exists a clique $C^\star = \left\{v_1, \ldots, v_{T}\right\}$ of size $T$, then there exists a policy with a cumulative reward of at least $\sum_{t=1}^T (1+\eta t)$.
		Consider the policy $\widetilde{\pi}$ such that $\widetilde{\pi}(t) = a_{v_{t},t}$. It is easy to see that $\widetilde{N}_{a_{v_{t},t},t}=t$ for every $t \in [T]$. Hence, the reward of the policy $\widetilde{\pi}$ at time $t$ is: \[\mu_{a_{v_{t},t}}(\widetilde{N}_{a_{v_{t},t},t})=\min\{1+ \eta t, \frac{t}{t}  (1+ \eta t)\}=1+ \eta t. \]
		Thus, $J_{\widetilde{\boldsymbol{\mu}},\widetilde{\mathbf{G}},T}(\widetilde{\pi}) = \sum_{t=1}^T (1+\eta t)$ and the claim is proven.
		
	\textbf{only if.} We show that if there is a policy $\widetilde{\pi}$ such that $J_{\widetilde{\boldsymbol{\mu}},\widetilde{\mathbf{G}},T}(\widetilde{\pi}) \ge \sum_{t=1}^T (1+\eta t)$, then there exists a clique of size $T$.

		First, we observe that for each $t',t \in [T]$ it holds that:
		\begin{align}\label{eq:red1}
        \max_{t'\in [T]} \min\{1+ \eta t', \frac{t}{t'}  (1+ \eta t')\}=1+ \eta t.
        \end{align}
		This implies that, at any round $t$, the best obtainable reward is:
  \begin{align*}
      \max_{t'\in [T]} \max_{v \in V} \max_{l \le t} {\mu}_{a_{v,t'}}(l) &=  \max_{t'\in [T]} \max_{v \in V} \widetilde{\mu}_{a_{v,t'}}(t)\\
      &= \max_{t'\in [T]} \min \left\{1+ \eta t', \frac{t}{t'}  (1+ \eta t')\right\} \\
      &=\min \left\{1+ \eta t, \frac{t}{t}  (1+ \eta t)\right\}=  1+ \eta t.
  \end{align*}
		Since by assumption there is a policy with reward at least $\sum_{t=1}^T (1+\eta t)$, then there is a policy such that at each round $t\in [T]$ the reward is exactly $1+ \eta t$.
		
		Consider a round $t\in [T]$. Let $a_{v,t'}$ be the arm played by the policy at this round. It must be the case that: i) $t'=t$, otherwise: \[\mu_{a_{v,t'}}(\widetilde{N}_{a_{v,t'},t})\le \mu_{a_{v,t'}}(t)  < 1+\eta t\] by Equation~\eqref{eq:red1}, and ii) $\widetilde{N}_{a_{v,t'},t} =t$, otherwise: \[\mu_{a_{v,t'}}(\widetilde{N}_{a_{v,t'},t})\le \frac{t-1}{t}(1+\eta t)< 1+\eta t .\]
		
		Let $a_{v_t,t}$ be the arm chosen at round $t$.
		Then, each arm in $\{a_{v_t,t}\}_{t\in[T]}$, is chosen while having exactly $t-1$ triggers. By the definition of $\widetilde{\mathbf{G}}$ this directly implies that $\{v_t\}_{t=1}^T$ is a clique of size $T$.
	\end{proof}

\begin{restatable}[\algnameblockshort Estimator's Instantaneous Bias]{lemma}{lemmaBiasdet}
\label{lemma:lemmaBiasdet}
For every arm $i \in [k]$, every round $t>1$, let us define:
\begin{align*}
	 \bar{\mu}_i(t) &\coloneq \mu_i(t_{i,N_{i,t-1}}^I) + ( t-t_{i,N_{i,t-1}}^I)\frac{\mu_i(t_{i,N_{i,t-1}}^I) - \mu_{i}(t_{i,N_{i,t-1}-1}^I)}{t_{i,N_{i,t-1}}^I- t_{i,N_{i,t-1}-1}^I} ,
\end{align*}
then, $\bar{\mu}_i(t) \ge \mu_i(t_{i,N_{i,t-1}}^I)$ and, if $N_{i,t-1} \ge 2$ it holds that:
\begin{align*}
	 \bar{\mu}_i(t)   - \mu_i(\widetilde{N}_{i,t}) \le (t-t_{i,N_{i,t}-1}^I)\gamma_{i}(t_{i,N_{i,t-1}-1}^I).
\end{align*}
\end{restatable}
\begin{proof}
    Let us start by observing the following equality holding:
    \begin{align*}
        \mu_i(\widetilde{N}_{i,t}) = \mu_i(t_{i,N_{i,t-1}}^I) + \sum_{j=t_{i,N_{i,t-1}}^I}^{\widetilde{N}_{i,t}-1} \gamma_i(j).
    \end{align*}
    We have:
    \begin{align}
        \mu_i(\widetilde{N}_{i,t}) &= \mu_i(t_{i,N_{i,t-1}}^I) + \sum_{j=t_{i,N_{i,t-1}}^I}^{\widetilde{N}_{i,t}-1} \gamma_i(j) \notag\\
        & \le \mu_i(t_{i,N_{i,t-1}}^I) + (\widetilde{N}_{i,t}-t_{i,N_{i,t-1}}^I) \gamma_i(t_{i, N_{i,t-1}-1}^I) \label{:bb:000}\\
        & \le \mu_i(t_{i,N_{i,t-1}}^I) + (t-t_{i,N_{i,t-1}}^I) \gamma_i(t_{i, N_{i,t-1}-1}^I), \label{:bb:001}
    \end{align}
    where line~\eqref{:bb:000} follows from Assumption~\ref{ass:rising}, and line~\eqref{:bb:001} is obtained from observing that $\widetilde{N}_{i,t} \le t$. Concerning the bias, when $N_{i,t-1} \ge 2$, we have:
    \begin{align}
     \bar{\mu}_i(t) - \mu_i(\widetilde{N}_{i,t}) &\le \mu_i(t_{i,N_{i,t-1}}^I) - \mu_i(\widetilde{N}_{i,t}) + ( t-t_{i,N_{i,t-1}}^I)\frac{\mu_i(t_{i,N_{i,t-1}}^I) - \mu_{i}(t_{i,N_{i,t-1}-1}^I)}{t_{i,N_{i,t-1}}^I- t_{i,N_{i,t-1}-1}^I} \\
     &\le ( t-t_{i,N_{i,t-1}}^I)\frac{\mu_i(t_{i,N_{i,t-1}}^I) - \mu_{i}(t_{i,N_{i,t-1}-1}^I)}{t_{i,N_{i,t-1}}^I- t_{i,N_{i,t-1}-1}^I} \label{:bb:002}\\
     &\le ( t-t_{i,N_{i,t-1}}^I)\gamma_i(t_{i, N_{i,t-1}-1}^I) \label{:bb:003},
    \end{align}
where line~\eqref{:bb:002} follows from observing that $\mu_i(t_{i,N_{i,t-1}}^I) \le \mu_i(\widetilde{N}_{i,t})$, and line~\eqref{:bb:003} derives from bounding $\frac{\mu_i(t_{i,N_{i,t-1}}^I) - \mu_{i}(t_{i,N_{i,t-1}-1}^I)}{t_{i,N_{i,t-1}}^I - t_{i,N_{i,t-1}-1}^I} \le \gamma_{i}(t_{i,N_{i,t-1}-1}^I)$ thanks to Assumption~\ref{ass:rising}.
\end{proof}

\TheoremdeterministicBlock*
\begin{proof}
Let $C_{\boldsymbol{\nu}, \mathbf{G}, T}^* \in \mathcal{C}_{\mathbf{G}}$ be the optimal clique of the instance. We analyze the following expression:
\begin{align*}
    R_{\boldsymbol{\nu}, \mathbf{G}, T}({\text{\algnameblockshort}}) = \sum_{t=1}^T \mu_{i_t^*}(t) - \mu_{I_t}(\widetilde{N}_{I_t,t}),
\end{align*}
where $i_t^* \in \argmax_{i \in C_{\boldsymbol{\nu}, \mathbf{G}, T}^*} \mu_i(t)$ for all $t\in[T]$. Then, we can decompose the regret in two meaningful components:
\begin{align}
    R_{\boldsymbol{\nu}, \mathbf{G}, T} (\text{\algnameblockshort}) &= \sum_{t=1}^T \mu_{i_t^*}(t) \pm \bar{\mu}_{I_t}(t) - \mu_{I_t}(\widetilde{N}_{I_t,t}) \notag \\
    &\le \sum_{t=1}^T \min\{1,\bar{\mu}_{I_t}(t) - \mu_{I_t}(\widetilde{N}_{I_t,t})\} \label{:tbr:000}\\ 
    &\le \sum_{t=1}^T\min\{1,(t-t_{I_t,N_{I_t,t-1}}^I)\gamma_{I_t}(t_{I_t,N_{I_t,t-1}-1}^I)\}  \label{:tbr:001}\\
    &= \sum_{t=1}^T\min\{1,(t\pm t_{I_t, N_{I_t,t}} ^I-t_{I_t,N_{I_t,t-1}}^I)\gamma_{I_t}(t_{I_t,N_{I_t,t-1}-1}^I)\} \notag\\
    &\le \sum_{t=1}^T \min\{1,(t - t_{I_t, N_{I_t,t}} ^I)\gamma_{I_t}(t_{I_t,N_{I_t,t-1}-1}^I)\} +\\
    &\qquad + \sum_{t=1}\min\{1,(t_{I_t, N_{I_t,t}} ^I-t_{I_t,N_{I_t,t-1}}^I)\gamma_{I_t}(t_{I_t,N_{I_t,t-1}-1}^I)\}  \label{:tbr:002}\\
    &= 4k + \underbrace{\sum_{C_m \in \mathcal{C}_{\mathbf{G}}}\sum_{i \in C_m} \sum_{j=3}^{N_{j,T}} \min\{1, (t-t_{i,j}^I)\gamma(t_{i,j-2}^I)\}}_\text{(a)} + \notag \\
    &\qquad +\underbrace{\sum_{C_m \in \mathcal{C}_{\mathbf{G}}}\sum_{i \in C_m} \sum_{j=3}^{N_{j,T}} \min\{1, (t_{i,j}^I-t_{i,j-1}^I)\gamma(t_{i,j-2}^I)\}}_\text{(b) } \notag,
\end{align}
where lines~\eqref{:tbr:000} and~\eqref{:tbr:001} follow from Lemma~\ref{lemma:lemmaBiasdet}, line~\eqref{:tbr:002} from the fact that $\min\{1,x+y\} \le \min\{1,x\} + \min\{1,y\}$ for any $x,y \ge 0$.

These two terms represent the rested and the restless contribution to the regret, and we can bound them using similar techniques as in~\citep{metelli2022stochastic}.
\end{proof}
\begin{remark}[Regret Bound in Rested and Restless Rising Bandits]
\label{rem:coincidence_rising}
    When we are in a purely rested (resp. restless) scenario, the contribution term associated to the restless (resp. rested) scenario vanish, and we get the same regret orders from~\citep{metelli2022stochastic}. In particular, we can avoid splitting the minimum in Equation~\eqref{:tbr:002} and instead notice that in a rested setting we have $t-t_{I_t, N_{I_t,t-1}}^I = t - N_{I_t,t-1}$, and thus we can bound the cumulative regret as we bound the term $\text{(a)}$. Instead, in a restless setting we have $t-t_{I_t, N_{I_t,t-1}}^I = t - t_{I_t, N_{I_t,t-1}}$, and thus we can bound the cumulative regret as we bound the term $\text{(b)}$. 
\end{remark}

\generaldeterministicBound*
\begin{proof}
    The theorem can be proved by showing that estimator's bias is always larger when internal times are decreased. For every arm $i\in[k]$ we define:
    \begin{equation}
        f_i(t;~ x,y) = \mu_i(x) + (t-x)\frac{\mu_i(x)-\mu_i(y)}{x-y},    \end{equation}
    for every triplet of natural numbers $y\le x \le t \le T$.
    Note that $\bar{\mu}_i(t) = f_i(t; ~t_{i, N_{i,t-1}}^I, t_{i, N_{i,t-1}-1}^I)$, so if we can show that $f_i$ is decreasing in both $x$ and $y$, we can prove the claim.
    We start with the second argument: fix $t$ and $x$, then for any $y$:
    \begin{align}
        f_i(t; x, y)-f_i(t; x, y-1) &= (t-x)\left(\frac{\sum_{j=y}^{x-1} \gamma_i(j)}{x-y}-\frac{\sum_{j=y-1}^{x-1}\gamma_i(j)}{x-y+1}\right) \notag \\
        &= \frac{\sum_{j=y}^{x-1}\gamma_i(j)-(x-y)\gamma_i(y-1)}{(x-y)(x-y+1)} \le 0 \label{:gdb:000},
    \end{align}
    where line~\eqref{:gdb:000} follows from Assumption~\ref{ass:rising}.
    With slightly more calculations we show that $f_i$ is also decreasing in the first argument, fix $t$ and $y$, then for any $x$:
    \begin{align}
        f_i(t;~ x, y)-&f_i(t; ~x-1, y) \le 0 \label{:gbd:001}.
    \end{align}
    
    Now we observe that, for every $i\in [k]$ and every $t \in [T]$, we have $t_{i,N_{i,t}}^I \ge t_{i,N_{i,t}}^{I,L}$. This is a consequence of Definition~\ref{def:block_submat}, since:
    $$
    t_{i,N_{i,t}}^I - t_{i,N_{i,t}}^{I,L} = \sum_{j=1}^{t} (G_{I_t, i}-\bar{G}_{I_t, i}^L) \ge 0.
    $$
    As a consequence of this, we have:
    \begin{equation}
    \label{eq:estimator_hierarchy}
    f_i(t;~t_{i,N_{i,t-1}}^I, t_{i,N_{i,t}-1}^I) \le f_i(t;~t_{i,N_{i,t-1}}^{I,L}, t_{i,N_{i,t}-1}^{I,L}),
    \end{equation}
    and 
    \begin{equation}
    \label{eq:reward_hierarchy}
        \mu_i(t_{i,N_{i,t}}^I) \ge \mu_i(t_{i,N_{i,t}}^{I,L}).        
    \end{equation}
   The proof can be concluded in the same way as for Theorem~\ref{thr:det_block_bound}.
\end{proof}

\begin{restatable}[Estimator's Instantaneous Bias]{lemma}{lemmaBias}
\label{lemma:lemmaBias}
For every arm $i \in [k]$, every round $t\in [T]$, and window width $1 \le h \le \left\lfloor \frac{N_{i,t-1}}{2} \right\rfloor$, let us define:
\begin{align*}
	 \widetilde{\mu}_i^{h}(t) &\coloneq \frac{1}{h} \sum_{l=N_{i,t-1}-h+1}^{N_{i,t-1}} \bigg(\mu_i(t_{i,l}^I) + ( t-l)\frac{\mu_i(t_{i,l}^I) - \mu_{i}(t_{i,l-h}^I)}{h} \bigg) ,
\end{align*}
otherwise if $h=0$, we set $\widetilde{\mu}_i^{h}(t)  \coloneqq +\infty$. Then, $\widetilde{\mu}_i^{h}(t) \ge \mu_i(t_{i,N_{i,t-1}})$ and, if $N_{i,t-1} \ge 2$ it holds that:
\begin{align*}
	 \widetilde{\mu}_i^{h }  (t)   - \mu_i(\widetilde{N}_{i,t}) \le \frac{(2t-2N_{i,t-1}+h-1)(t_{i,N_{i,t-1}}^I - t_{i,N_{i,t-1}-2h+1}^I)}{2h} \gamma_{i}(t_{i,N_{i,t-1}-2h+1}^I).
\end{align*}
\end{restatable}
\begin{proof}
	Let us start by observing the following equality holding for every $l \in \{2,\dots,N_{i,t-1}\}$:
	\begin{align*}
		\mu_i(\widetilde{N}_{i,t}) = \mu_i(t_{i,l}^I) + \sum_{j=t_{i,l}^I}^{\widetilde{N}_{i,t}-1} \gamma_i(j).
	\end{align*}
	By averaging over a window of length $h$, we obtain:
	\begin{align}
		\mu_i(\widetilde{N}_{i,t}) & = \frac{1}{h} \sum_{l=N_{i,t-1}-h+1}^{N_{i,t-1}} \left(\mu_i(t_{i,l}^I) +  \sum_{j=t_{i,l}^I}^{\widetilde{N}_{i,t}-1} \gamma_i(j) \right) \notag\\
		& \le \frac{1}{h} \sum_{l=N_{i,t-1}-h+1}^{N_{i,t-1}} \left(\mu_i(t_{i,l}^I) +  (\widetilde{N}_{i,t} - t_{i,l}^I )\gamma_i(t_{i,l}^I-1) \right) \label{:prs:3001}\\
		& \le \frac{1}{h} \sum_{l=N_{i,t-1}-h+1}^{N_{i,t-1}} \left(\mu_i(t_{i,l}^I) +  \frac{\widetilde{N}_{i,t} - t_{i,l}^I }{t_{i,l}^I - t_{i,l-h}^I} \sum_{j={t_{i,l-h}^I}}^{t_{i,l}^I-1} \gamma_i(j) \right)  \label{:prs:3002}\\
		& \le \frac{1}{h} \sum_{l=N_{i,t-1}-h+1}^{N_{i,t-1}} \left(\mu_i(t_{i,l}^I) +  (t - l)\frac{\mu_i(t_{i,l}^I) - \mu_i(t_{i,l-h}^I) }{h}  \right) \eqqcolon \widetilde{\mu}_i^{h}(t),  \label{:prs:3003}
	\end{align}
	where lines~\eqref{:prs:3001} and~\eqref{:prs:3002} follow from Assumption~\ref{ass:rising}, and line~\eqref{:prs:3003} is obtained from observing that $t_{i,l}^I \ge l$, $\widetilde{N}_{i,t} \le t$ and $t_{i,l}^I- t_{i,l-h}^I \ge h$.
 
	 Concerning the bias, when $N_{i,t-1} \ge 2$, we have:
	 \begin{align}
	 \widetilde{\mu}_i^{h}(t)- \mu_i(\widetilde{N}_{i,t}) & = \frac{1}{h} \sum_{l=N_{i,t-1}-h+1}^{N_{i,t-1}} \left( \mu_i(t_{i,l}^I) +  (t-l)\frac{\mu_i(t_{i,l}^I) - \mu_{i}(t_{i,l-h}^I)}{h}  \right) - \mu_i(\widetilde{N}_{i,t})\notag \\
	 & \le \frac{1}{h} \sum_{l=N_{i,t-1}-h+1}^{N_{i,t-1}}   (t-l)\frac{\mu_i(t_{i,l}^I) - \mu_{i}(t_{i,l-h}^I)}{h}   \label{:prs:4001}\\
	  &= \frac{1}{h} \sum_{l=N_{i,t-1}-h+1}^{N_{i,t-1}}   (t-l)\frac{\mu_i(t_{i,l}^I) - \mu_{i}(t_{i,l-h}^I)}{t_{i,l}^I - t_{i,l-h}^I}  \frac{t_{i,l}^I - t_{i,l-h}^I}{h}   \notag\\
	  &\le \frac{1}{h} \sum_{l=N_{i,t-1}-h+1}^{N_{i,t-1}}   (t-l) \gamma_{i}(t_{i,l-h}^I)  \frac{t_{i,l}^I - t_{i,l-h}^I}{h}   \label{:prs:4004}\\
	  &\le \frac{t_{i,N_{i,t-1}}^I - t_{i,N_{i,t-1}-2h+1}^I}{h^2} \gamma_{i}(t_{i,N_{i,t-1}-2h+1}^I) \sum_{l=N_{i,t-1}-h+1}^{N_{i,t-1}}   (t-l)   \label{:prs:4005}\\
	  & = \frac{(2t-2N_{i,t-1}+h-1)(t_{i,N_{i,t-1}}^I - t_{i,N_{i,t-1}-2h+1}^I)}{2h} \gamma_{i}(t_{i,N_{i,t-1}-2h+1}^I), \label{:prs:4006}
	 \end{align}
	 where line~\eqref{:prs:4001} follows from observing that $\mu_i(t_{i,l}^I) \le \mu_i(\widetilde{N}_{i,t})$, line~\eqref{:prs:4004} derives from Assumption~\ref{ass:rising} and bounding $\frac{\mu_i(t_{i,l}^I) - \mu_{i}(t_{i,l-h}^I)}{t_{i,l}^I - t_{i,l-h}^I} \le \gamma_{i}(t_{i,l-h}^I)$, line~\eqref{:prs:4005} is obtained by bounding $t_{i,l}^I - t_{i,l-h}^I \le t_{i,N_{i,t-1}}^I - t_{i,N_{i,t-1}-2h+1}^I$ and $\gamma_{i}(t_{i,l-h}^I) \le \gamma_{i}(t_{i,N_{i,t-1}-2h+1}^I)$, and line~\eqref{:prs:4006} follows from computing the summation.
	
\end{proof}

\begin{restatable}[Bound on Estimator's Cumulative Bias for Block-Diagonal CMs]{lemma}{lemmaCumBiasBound}
\label{lemma:cumbiasbound}
\phantom{a} Let $(I_t)_{t \geq 1}$ be a sequence of actions.
For every action $i \in[k]$, every round $t\in [T]$, let window width $h_{i,t}=\lfloor \epsilon N_{i,t-1} \rfloor$. Let $\mathbf{G} \in \mathbb{B}_{\widetilde{k}}$ be a block diagonal matrix, then for every $q\in[0,1]$, we have:
\begin{align*}
    \sum_{t=1}^T & \min \left\{ 1, \widetilde{\mu}_{I_t}^{h_{I_t,t}}(t) -\mu_{I_t}(\widetilde{N}_{I_t,t}) \right\} \le \\ 
    &\le 2k + \bar{k}_1 T^q \left \lceil \frac{1}{1-2\epsilon}\right \rceil \Upsilon_{\boldsymbol{\nu}}\left(\left\lceil (1-2\epsilon)\frac{T}{\bar{k}_1}\right\rceil, q\right) + \\
    & \ \ +T^{\frac{2q}{1+q}} (1 + \log(\epsilon T))^\frac{q}{1+q} \left\lceil \frac{1}{\epsilon} \right\rceil \left\lceil \frac{1}{1-2\epsilon} \right\rceil \sum_{C_m \in \mathcal{C}_{\mathbf{G}}: |C_m| > 1} \!\!\!\! |C_m| \Upsilon_{\boldsymbol{\nu}}\left( \left\lceil(1-2\epsilon) \frac{T}{|C_m|} \right\rceil,  q \right) ^{\frac{1}{1+q}},
\end{align*}
where $\mathcal{C}$ is the set of blocks of matrix $\mathbf{G}$, and $\bar{k}_1 \le k$ is the number of blocks of size $1$.
\end{restatable}
\begin{proof}
The statement can be proven by decomposing over the cliques and then over the arms, splitting cliques with only one arm from the others:
\begin{align*}
\sum_{t=1}^T \min\left\{1, \widetilde{\mu}_{I_t}^{h_{I_t,t}}(t) - \mu_{I_t}(\widetilde{N}_{I_t,t})\right\} \le & \ 2k \ +  \underbrace{\sum_{\substack{C_m \in \mathcal{C}_{\mathbf{G}}: |C_m| = 1 \\ C_m = \{i\}}}\sum_{j=3}^{N_{i,T}}\min \left\{1, \widetilde{\mu}_{i}^{h_{i,t_{i,j}}}(t_{i,j}) - \mu_{i}(j) \right\}}_{\text{(a)}} + \\ & + \!\!  \underbrace{\sum_{C_m \in \mathcal{C}_{\mathbf{G}}: |C_m| > 1} \sum_{i\in C_m}\sum_{j=3}^{N_{i,T}}\min \left\{1, \widetilde{\mu}_{i}^{h_{i,t_{i,j}}}(t_{i,j}) - \mu_{i}(t_{i,j}^I) \right\}}_{\text{(b)}}.\notag\\
\end{align*}
The two terms can be bound in a similar way as in~\citep{metelli2022stochastic}, as the rested and the restless component, respectively.
\end{proof}

\regretUBblock*

\begin{proof}
Let us define the good events $\mathcal{E}_t = \bigcap_{i \in [k]} \mathcal{E}_{i,t}$ that correspond to the event in which all confidence intervals hold:
\begin{align*} 
\mathcal{E}_{i,t} \coloneqq \left\{ \left| \widehat{\mu}_i^{h_{i,t}}(t) - \widetilde{\mu}_i^{h_{i,t}}(t)\right| \le \beta^{h_{i,t}}_i(t)\right\} \qquad \forall i \in [T], \,i \in [k].
\end{align*}

We have to analyze the following expression:
\begin{align*}
	R_{\boldsymbol{\nu},\mathbf{G}, T}(\text{\algnamegenshort}) =  \mathbb{E} \left[\sum_{t=1}^T  \mu_{i^*_t}(t) - \mu_{I_t}(t) \right] ,
\end{align*}
where $i^*_t \in \argmax_{i \in C_{\boldsymbol{\nu},\mathbf{G}, T}^*}\mu_i(t)$ for all $t\in [T]$. 
 We decompose according to the good events $\mathcal{E}_t$:
\begin{align*}
R_{\boldsymbol{\nu},\mathbf{G}, T}(\text{\algnamegenshort}) & = \sum_{t=1}^T \mathbb{E} \left[\left( \mu_{i^*_t}(t) - \mu_{I_t}(t) \right)\mathds{1}\{\mathcal{E}_t\} \right] + \sum_{t=1}^T \mathbb{E} \left[\left( \mu_{i^*_t}(t) - \mu_{I_t}(t) \right)\mathds{1}\{\lnot\mathcal{E}_t\}\right] \\
& \le \sum_{t=1}^T \mathbb{E} \left[\left( \mu_{i^*_t}(t) - \mu_{I_t}(t) \right)\mathds{1}\{\mathcal{E}_t\}\right] + \sum_{t=1}^T \mathbb{E} \left[\mathds{1}\{\lnot\mathcal{E}_t\}\right], 
\end{align*}
where we exploited $\mu_{i^*_t}(t) - \mu_{I_t}(t) \le 1$ in the inequality.
The second summation can be bounded using standard arguments, recalling that $\alpha > 2$:
\begin{align*}
\sum_{t=1}^T \mathbb{E} \left[\mathds{1}\{\lnot\mathcal{E}_t\}\right] &\le 1 + \sum_{i \in [k]} \sum_{t=2}^T \mathbb{P}\left( \lnot \mathcal{E}_{i,t}\right) \\
& \le 1+ \frac{2k}{\alpha-2}.
\end{align*}
where the first inequality is obtained with $\mathbb{P}(\lnot \mathcal{E}_{1}) \le 1$ and a union bound over $[k]$. Recalling $\mathbb{P}(\lnot\mathcal{E}_{i,t})$ was bounded in Lemma~\ref{lemma:lemmaConcentration}, we bound the summation with the integral and obtain the second inequality.

The rest of the analysis can be conducted under the good event $\mathcal{E}_t$, recalling that $B_i(t) \equiv \widehat{\mu}_i^{h_{i,t}}(t) + \beta^{h_{i,t}}_i(t)$. Let $t\in[T]$, and we exploit the optimism, \ie $B_{i^*_t}(t) \le B_{I_t}(t)$:
\begin{align*}
	\mu_{i^*}(t) - \mu_{I_t}(t) + B_{I_t}(t) - B_{I_t}(t) & \le  \min \left\{ 1, \underbrace{\mu_{i^*_t}(t) - B_{i^*_t}(t)}_{\le 0} + B_{I_t}(t) - \mu_{I_t}(t) \right\} \\
	& \le \min \left\{ 1, B_{I_t}(t) - \mu_{I_t}(t)\right\}.
\end{align*}
Now, we work on the term inside the minimum:
\begin{align}
	B_{I_t}(t) - \mu_{I_t}(t) &= \widehat{\mu}_{I_t}^{h_{I_t,t}}(t) + \beta^{h_{I_t,t}}_{I_t}(t)  - \mu_{I_t}(t) \label{lprl:001}\\
	& \le \underbrace{\widetilde{\mu}_{I_t}^{h_{I_t,t}}(t)  - \mu_{I_t}(t)}_{\text{(a)}} + \underbrace{2\beta^{ h_{I_t,t}}_{I_t}(t)}_{\text{(b)}},\label{lprl:002}
\end{align}	
where line~\eqref{lprl:001} follows from the definition of $B_i(t)$ and line~\eqref{lprl:002} from the good event $\mathcal{E}_t$. We make use of Lemma~\ref{lemma:cumbiasbound} and Lemma~\ref{lemma:cumvariancebound} to bound the summations over $t$ of (a) and (b), respectively.

Putting all together, we obtain:
\begin{align*}
R&{}_{\boldsymbol{\nu},\mathbf{G}, T}(\text{\algnamegen}) \\ 
& \le 1+ \frac{2k}{\alpha-2} + 5k + \frac{k}{\epsilon}   + \frac{3k}{\epsilon}(2\sigma T)^{\frac{2}{3}} \left( 10 \alpha \log T \right)^{\frac{1}{3}} + \\
& \quad + T^{\frac{2q}{1+q}} (1 + \log(\epsilon T))^{\frac{q}{1+q}}\left\lceil \frac{1}{\epsilon} \right\rceil \left\lceil \frac{1}{1-2\epsilon} \right\rceil k \Upsilon_{\bm{\mu}}\left( \left\lceil(1-2\epsilon) \frac{T}{k} \right\rceil,  q \right)^{\frac{1}{1+q}} + \\
& \quad + 2k + \bar{k}_1 T^q \left \lceil \frac{1}{1-2\epsilon}\right \rceil \Upsilon_{\boldsymbol{\nu}}\left(\left\lceil (1-2\epsilon)\frac{T}{\bar{k}_1}\right\rceil, q\right)  +\notag\\
& \quad +T^{\frac{2q}{1+q}} (1 + \log(\epsilon T))^\frac{q}{1+q} \left\lceil \frac{1}{\epsilon} \right\rceil \left\lceil \frac{1}{1-2\epsilon} \right\rceil \sum_{C_m \in \mathcal{C}_{\mathbf{G}}: |C_m| > 1} |C_m| \Upsilon_{\boldsymbol{\nu}}\left( \left\lceil(1-2\epsilon) \frac{T}{|C_m|} \right\rceil,  q \right) ^{\frac{1}{1+q}}.
\end{align*}
\end{proof}


\begin{restatable}[Bound on Estimator's Cumulative Bias for General Matrices]{lemma}{lemmaCumBiasBoundGen}
\label{lemma:cumbiasboundgen}
Let $(I_t)_{t\in[T]}$ be any sequence of actions. For every action
$i\in[k]$ and every round $t\in[T]$, let the window width be
$h_{i,t}=\lfloor \epsilon N_{i,t-1}\rfloor$, with
$\epsilon\in(0,1/2)$. Let $\mathbf{G}\in\{0,1\}^{k\times k}$ be a general
connectivity matrix, and let $\bar{\mathbf{G}}^U$ be a minimal block super-matrix
of $\mathbf{G}$, with clique partition $\mathcal{C}_{\bar{\mathbf{G}}^U}$. Let:
\[ \bar k_1^U \coloneqq
    \bigl|\{C\in \mathcal{C}_{\bar{\mathbf{G}}^U}: |C|=1\}\bigr|, \]
be the number of singleton cliques of $\bar{\mathbf{G}}^U$. Then, for every
$q\in[0,1]$, it holds that:
\begin{align*}
\sum_{t=1}^T
& \min\left\{
1, \widetilde\mu_{I_t}^{h_{I_t,t}}(t) - \mu_{I_t}(\widetilde N_{I_t,t}) \right\} \\
& \le 2k + \bar k_1^U T^q \left\lceil \frac{1}{1-2\epsilon}\right\rceil
\Upsilon_\nu \left( \left\lceil \frac{(1-2\epsilon)T}{\bar k_1^U} \right\rceil, q \right) \\
& \ + T^{\frac{2q}{1+q}} (1 + \log(\epsilon T))^\frac{q}{1+q} \left\lceil \frac{1}{\epsilon} \right\rceil \left\lceil \frac{1}{1-2\epsilon} \right\rceil \cdot \sum_{\substack{C_m^U \in \mathcal{C}_{\bar{\mathbf{G}}^U} |C_m^U|>1}} \!\!\!\! |C_m| \Upsilon_{\bm{\nu}}\left( \left\lceil(1-2\epsilon) \frac{T}{|C_m|} \right\rceil,  q \right) ^{\frac{1}{1+q}}\!\!.
\end{align*}
When $\bar k_1^U=0$, the singleton-clique term is taken to be zero.
\end{restatable}

\begin{proof}
The proof follows similar steps as Lemma~\ref{lemma:cumbiasbound}. We decided to split arms based on their degree; in particular, we bound separately the bias due to arms having a degree of $1$ (\ie they are only triggered by themselves).

\begin{align*}
&\sum_{t=1}^T \min\left\{1, \widetilde{\mu}_{I_t}^{h_{I_t,t}}(t) - \mu_{I_t}(\widetilde{N}_{I_t,t})\right\} \\
&\le 2k + \!\!\! \underbrace{\sum_{\substack{i \in [k] \\ \text{deg}^-({i})=1}} \!\! \sum_{j=3}^{N_{i,T}}\min \left\{1, \widetilde{\mu}_{i}^{h_{i,t_{i,j}}}(t_{i,j}) - \mu_{i}(j) \right\}}_{\text{(a)}}+ \!\! \underbrace{\sum_{\substack{i \in [k] \\ \text{deg}^-({i})>1}} \!\! \sum_{j=3}^{N_{i,T}}\min \left\{1, \widetilde{\mu}_{i}^{h_{i,t_{i,j}}}(t_{i,j}) - \mu_{i}(t_{i,j}^I) \right\}}_{\text{(b)}}.\notag
\end{align*}

As a consequence of Definition~\ref{def:block_submat}, we observe that:
    \begin{equation*}
    t_{i,N_{i,t}}^I - t_{i,N_{i,t}}^{I,U} = \sum_{j=1}^{t} (G_{I_t, i}-\bar{\mathbf{G}}_{I_t, i}^U) \le 0.
    \end{equation*}
As a consequence of this, we have that, for every $i \in [k]$ and for every $t \in [T]$:
\begin{equation}
\label{eq:pulls_hierarchy}
    \widetilde{N}_{i,t} \le \widetilde{N}_{i,t}^U,
\end{equation}
where $ \widetilde{N}_{i,t}^U \coloneqq \mathbf{e}_i^\top(\bar{\mathbf{G}}^U)^\top\mathbf{N}_t$.
Then, following similar steps as in~\citep{metelli2022stochastic}, we can bound the two components separately and make the dependency on the upper block-diagonal matrix explicit.
\end{proof}

\regretUBgeneric*
\begin{proof}
    The proof follows similar steps of the proof of Theorem~\ref{thr:regret_block}, but uses Lemma~\ref{lemma:cumbiasboundgen} (instead of Lemma~\ref{lemma:cumbiasbound}) to bound cumulative estimator's bias.

As in Theorem~\ref{thr:regret_block}, we decompose the regret in two components and instead make use of Lemma~\ref{lemma:cumbiasboundgen} and Lemma~\ref{lemma:cumvariancebound} to bound the summations over $t$ of the two components, respectively. Putting all together, we obtain:
\begin{align*}
R_{\boldsymbol{\nu},\mathbf{G}, T} & (\text{\algnamegen}) \le 1+ \frac{2k}{\alpha-2} + 5k + \frac{k}{\epsilon}+ \frac{3k}{\epsilon}(2\sigma T)^{\frac{2}{3}} \left( 10 \alpha \log T \right)^{\frac{1}{3}} + \notag \\ & + 2k + \bar{k}_1 T^q \left \lceil \frac{1}{1-2\epsilon}\right \rceil \Upsilon_{\boldsymbol{\nu}}\left(\left\lceil (1-2\epsilon)\frac{T}{\bar{k}_1}\right\rceil, q\right)+  \notag\\
&  +T^{\frac{2q}{1+q}} (1 + \log(\epsilon T))^\frac{q}{1+q} \left\lceil \frac{1}{\epsilon} \right\rceil \left\lceil \frac{1}{1-2\epsilon} \right\rceil \cdot \sum_{\substack{C_m^U \in \mathcal{C}_{\bar{\mathbf{G}}^U} \\ |C_m^U|>1}}|C_m| \Upsilon_{\bm{\nu}}\left( \left\lceil(1-2\epsilon) \frac{T}{|C_m|} \right\rceil,  q \right) ^{\frac{1}{1+q}}.
\end{align*}
\end{proof}

\subsection{Technical Lemmas}
\label{apx:lemmas}

\begin{lemma}[Lemma~C.1 of~\citealt{metelli2022stochastic}]\label{lemma:sumWithFloor}
	Let $M \ge 3$, and let $f: \mathbb{N} \rightarrow \mathbb{R}$, and $\beta \in (0, 1)$. Then it holds that:
	\begin{align*}
		\sum_{j=3}^{M} f(\lfloor \beta j \rfloor) \le \left\lceil \frac{1}{\beta} \right\rceil \sum_{l=\lfloor 3 \beta \rfloor}^{\lfloor \beta M \rfloor} f(l).
	\end{align*}
\end{lemma}

\begin{lemma}[Lemma~C.2 of~\citealt{metelli2022stochastic}]\label{lemma:boundUpsilonMax}
	Under Assumption~\ref{ass:rising}, it holds that:
	\begin{align*}
		\max_{\substack{(N_{i,T})_{i \in [k]} \\ N_{i,T} \ge 0, \sum_{i \in [k]}N_{i,T} = T}} \; \; \sum_{i \in [k]} \sum_{l=1}^{N_{i,T}-1} \gamma_i(l)^q \le k \Upsilon_{\bm{\nu}}\left( \left\lceil \frac{T}{k} \right\rceil, q \right).
	\end{align*}
\end{lemma}

\lemmaConcentration*
\begin{proof}
    Using a Doob's \textit{optional skipping} argument~\citep{doob1953stochastic,bubeck2008online}, and noting that, at round $t$, $t_{i,l}^I$ is a stopping time for every arm $i\in [k]$ and pull number $l \in \{1, \ldots, N_{i,t-1}\}$ w.r.t.\ the filtration $\mathcal{F}_{\tau-1} = \sigma(I_1, X_1, \ldots, I_{\tau-1}, X_{\tau-1}, I_{\tau})$, we can proceed to prove this lemma as in~\citep{metelli2022stochastic} also for \settingnameshort.
\end{proof}

\begin{restatable}[Bound on Estimator's Variance, Theorem 4.4 of~\citealt{metelli2022stochastic}]{lemma}{lemmaVarianceBound}
\label{lemma:cumvariancebound}
Let $(I_t)_{t\in[T]}$ be a sequence of actions such that:
\begin{equation}
    \left| \widehat{\mu}_{I_t}^{h_{I_t,t}}(t) - \widetilde{\mu}_{I_t}^{h_{I_t,t}}(t) \right| \le\beta^{h_{I_t,t}}_{I_t}(t,t^{-\alpha}), ~\forall t \in [T],
\end{equation}
where $\alpha >2$.
For every action $i \in[k]$, every round $t\in[T]$, let window width $h_{i,t}=\lfloor \epsilon N_{i,t-1} \rfloor$, then, we have:
\begin{equation}
    \sum_{t=1}^T \min\left\{1, 2\beta_{I_t}^{h_{I_t,t}}(t,t^{-\alpha})\right\} \le k\left(3+\frac{1}{\epsilon}\right)+\frac{3k}{\epsilon}(2\sigma T)^\frac{2}{3}(10\alpha \log T)^\frac{1}{3}.
\end{equation}
\end{restatable}

%% file: content/08apx_rotting.tex
\section{Proofs on Rotting Bandits}
\label{apx:rotting}

\theoremHardnessRotting*

\begin{proof}
We reduce from a decision problem related to finding independent sets in graphs. In particular, given a graph $(V,E)$ and $\widetilde{M} \in \mathbb{N}$, it is NP-hard to determine if there exists an independent set of size $\widetilde{M}$~\citep{karp1972reducibility}. In the following, we design an instance of our problem such that the reward of the optimal policy is at least $T$ if and only if there exists an independent set of size $\widetilde{M}=T$. 
	
\textbf{Construction.}
Given a graph $(V,E)$, we build an instance such that the horizon is $T$. Our set of actions can be constructed by assigning an action to every node, \ie $\mathcal{A}= \{a_{v}\}_{v\in V}$.  We define the matrix $\widetilde{\mathbf{G}}$ is such that for any $v,v' \in V$, it holds $G_{a_{v},a_{v'}}=1$ if $(v,v')\in E$, and $G_{a_{v},a_{v'}}=0$ otherwise. Finally, for each arm $a_{v} \in \mathcal{A}$, the reward is deterministic and evolves as ${\mu}_{a_{v,t}}(n)= \max\{2-n,0\}$. We call $\widetilde{\boldsymbol{\nu}}$ the set of these functions. It is easy to see that the \settingnameshort instance $(\widetilde{\boldsymbol{\nu}}, \widetilde{\mathbf{G}}, T)$ satisfies \cref{ass:rotting}.

\textbf{if.}
We show that if there exists an independent set $I^\star = \left\{v_1, \ldots, v_{T}\right\}$ of size $T$, then there exists a policy with a cumulative reward of at least $T$.
Consider the policy $\widetilde{\pi}$ such that $\widetilde{\pi}(t) = a_{v_{t}}$. It is easy to see that $\widetilde{N}_{a_{v_{t}},t}=1$ for every $t \in [T]$. Hence, the reward of the policy $\widetilde{\pi}$ at time $t$ is: \[\mu_{a_{v_{t}}}(\widetilde{N}_{a_{v_{t}},t})=1. \]
Thus, $J_{\widetilde{\boldsymbol{\mu}},\widetilde{\mathbf{G}},T}(\widetilde{\pi}) = T$ and the claim is proven.
	
\textbf{only if.} 
We show that if there is a policy $\widetilde{\pi}$ such that $J_{\widetilde{\boldsymbol{\mu}},\widetilde{\mathbf{G}},T}(\widetilde{\pi}) \ge T$, then there exists an independent set of size $T$. First, we observe that at any round $t$ the best obtainable reward is $1$. Since, by assumption, there is a policy with a reward of at least $T$; then there is a policy such that at each round $t\in [T]$, the reward is exactly $1$.
	
Let $a_{v_t}$ be the arm played by the policy at round $t \in [T]$. Then, consider a round $t \in [T]$.  Since the reward of the arm $a_{v_t}$ must be $1$, it must be the case that $\mu_{a_{v_{t}}}(\widetilde{N}_{a_{v_{t}},t})=1$ and $\widetilde{N}_{a_{v_{t}},t}=1$. 
By the definition of $\widetilde{\mathbf{G}}$ this directly implies that $\{v_t\}$ is not connected to any $v_{t'}$, $t' < t$, and that $v_{t'}\neq v_t$ for any $t' < t$. Hence, $\{v_t\}_{t \in [T]}$ is an independent set of size $T$, proving the claim.
\end{proof}

\optimalityRotting*
\begin{proof}
   For every Rotting GTB instance, we create an alternative instance which is better, in terms of total cumulative reward, than the original instance. Then we show that playing greedy in the original instance yields the same cumulative reward of the optimal policy from the alternative instance.
    
    For  each clique $C_m \in \mathcal{C}_\mathbf{G}$, we substitute the reward function of every arm $i \in C_m$ with $\mu_i^*(n) = \max_{i \in C_m} \mu_i(n)$ for every $n \in [T]$. This way, whenever an action is chosen it is guaranteed to always yield the same reward as any other possible action inside the same clique. We create an alternative instance $(\widetilde{\boldsymbol{\nu}}, \widetilde{\mathbf{G}}, T)$ by collapsing all the actions inside the same clique into a single meta-action, resulting in a $\widetilde{k}$-armed rested rotting bandit problem, where the set of actions corresponds to the set of cliques of the original instance. We use Proposition 2 of~\citep{heidari2016tight} to get that the optimal policy in the alternative instance is to play, at every round, the action with the highest instantaneous reward. Such policy achieves a total reward, in the alternative instance, of:
    \begin{equation*}
        J_{\widetilde{\boldsymbol{\nu}}, \widetilde{\mathbf{G}}, T}^* 
        = \sum_{C_m \in \mathcal{C}_\mathbf{G}} \sum_{n=1}^{N_{C_m,T}^*} \max_{i \in C_m} \mu_i(n),
    \end{equation*}
    We now show that playing the greedy policy in the original instance yields an equal total cumulative reward. Playing greedily in the original instance we get:
    \begin{align*}
        J_{\boldsymbol{\nu}, \mathbf{G}, T}^* &= \sum_{t=1}^T \max_{i \in [k]}\mu_i(\widetilde{N}_{i,t}^*) \\
        &= \sum_{C_m \in\mathcal{C}_\mathbf{G}} \sum_{t=1}^T \mathbbm{1}_{\{{I_t^* \in C_m}\}} \max_{i \in [k]}\mu_i(\widetilde{N}_{i,t}^*) \\
        &= \sum_{C_m \in\mathcal{C}_\mathbf{G}} \sum_{t=1}^T \mathbbm{1}_{\{{I_t^* \in C_m}\}} \max_{i \in C_m}\mu_i(\widetilde{N}_{i,t}^*) \\
        &=  \sum_{C_m \in \mathcal{C}_\mathbf{G}} \sum_{n=1}^{N_{C_m,T}^*} \max_{i \in C_m} \mu_i(n) \\
        &= J_{\widetilde{\boldsymbol{\nu}}, \widetilde{\mathbf{G}}, T}^* \ge J_{\boldsymbol{\nu},\mathbf{G}, T}(\pi), \qquad \forall \pi.
    \end{align*}
    The performance of the optimal policy in the alternative instance is matched by the greedy policy played in the original instance. The proof is concluded by observing that the optimal total cumulative reward of the alternative instance cannot be lower than the total reward of any policy $\pi$ in the original instance, since the alternative instance has pointwise higher reward functions for every action.
\end{proof}

\subsection{Upper Bounding the Regret of \texttt{RAW-UCB}}
We start by defining the expectation version of the estimator defined in Equation~\eqref{eq:estimator_rotting} as $\bar{\mu}_i^h(t) \coloneqq \frac{1}{h}\sum_{s=1}^{t-1} \mathbbm{1}_{\{I_t = i ~\land~ N_{i,s}>N_{i,t-1}-h\}} \mu_{i}(\widetilde{N}_{i,s})$. Before moving on, we recall the following result, which also introduces the notion of \textit{good event} $\xi_t^\alpha$.
\begin{proposition}[Bound on the Probability of Bad Event, \citealt{seznec2020single}]
\label{prop:varbound}
    Let $\delta_t = 2t^{-\alpha}$, and:
    \begin{align*}
        \xi_t^\alpha \coloneqq \left\{ \forall i\in[k], ~\forall n\le t-1,~ \forall h \le n,~ | \widehat{\mu}_i^h(t)- \bar{\mu}_i^h(t)| \le c(h,\delta_t)\right\},
    \end{align*}
    for $c(h,\delta_t) \coloneqq \sqrt{2\sigma^2\log(2\delta_t^{-1})/h}$. Then:
    \begin{equation}
        \label{eq:goodevent_prob}
        \mathbb{P}\left(\bar{\xi_t^\alpha}\right) \le Kt^{2-\alpha}.
    \end{equation}
\end{proposition}
\begin{lemma}[Overestimation under the Good Event]
\label{lemma:ucb_gt}
Under $\xi_t^\alpha$, if action $I_t$ is selected by Algorithm~\ref{alg:rawucb}, for every $h \in [N_{i,t-1}]$ we have:
\begin{equation}
    \label{eq:overest_rotting}
    \bar{\mu}_{I_t}^h \ge \max_{i \in [k]} \mu_i(\widetilde{N}_{i,t-1}^\pi) - 2c(h,\delta_t),
\end{equation}
where $\widetilde{N}_{i,t-1}^\pi$ is the number of triggers of action $i$ provoked by playing with Algorithm~\ref{alg:rawucb} up until time $t$.
\end{lemma}
\begin{proof}
This proof is adapted from the one of Lemma 1 of~\citep{seznec2020single}.
    Let $h_{i,t}^{\min }\in \argmin_{h \le N_{i,t-1}} \widehat{\mu}_i^h(t) + c(h, \delta_t)$.
    
    Let $i_t^\pi \in \argmax_{i \in [k]} \mu_i(\widetilde{N}_{i,t-1}^\pi)$ be the best available action at time $t$. From the rotting assumption, we know that:
    $$
    \max_{i \in [k]} \mu_i(\widetilde{N}_{i,t-1}^\pi)  = \mu_{i_t^\pi}(\widetilde{N}_{i,t-1}^\pi) \le \bar{\mu}_{i_t^\pi}^1(t) \le \ldots \le \bar{\mu}_{i_t^\pi}^{h_{i_t^\pi,t}^{\min}}(t).
    $$
    Under $\xi_t^\alpha$, we have:
    $$
    \bar{\mu}_{i_t^\pi}^{h_{i_t^\pi,t}^{\min}}(t) \le \widehat{\mu}_{I_t}^{h_{I_t,t}^{\min}}(t) + c(h_{I_t,t}^{\min}, \delta_t).
    $$
    We now use the definition of $h_{I_t,t}^{\min}$:
    $$
     \widehat{\mu}_{I_t}^{h_{I_t,t}^{\min}}(t) + c(h_{I_t,t}^{\min}, \delta_t) \le \widehat{\mu}_{I_t}^{h}(t) + c(h, \delta_t).
    $$
    Again, we use $\xi_t^\alpha$:
    $$
     \widehat{\mu}_{I_t}^{h}(t) + c(h, \delta_t)\le \bar{\mu}_{I_t}^{h}(t) + 2c(h, \delta_t).
    $$

    Putting all together, we obtain the statement.
    \end{proof}

\regretUBblockROT*
\begin{proof}
Let us proceed to decompose the regret:
    \begin{align*}
        R_{\boldsymbol{\nu}, \mathbf{G}, T}(\texttt{RAW-UCB}) &= \sum_{t=1}^T \left( \mu_{i_t^*}(\widetilde{N}_{i_t^*, t}^*) - \mu_{I_t}(\widetilde{N}_{I_t, t}^\pi) \right) \\
        &= \sum_{t=1}^T \left( \mu_{i_t^*}(\widetilde{N}_{i_t^*, t}^*) - \mu_{I_t}(\widetilde{N}_{I_t, t}^\pi) \pm \max_{i \in C_{I_t}} \mu_{i}(\widetilde{N}_{I_t, t}^\pi) \right) \\
        &=\underbrace{\sum_{t=1}^T (\mu_{i_t^*}(\widetilde{N}_{i_t^*, t}^*) - \max_{i \in C_{I_t}} \mu_{i}(\widetilde{N}_{I_t, t}^\pi))}_{\text{(b)}} + \underbrace{\sum_{t=1}^T (\max_{i \in C_{I_t}} \mu_{i}(\widetilde{N}_{I_t, t}^\pi)-\mu_{I_t}(\widetilde{N}_{I_t, t}^\pi))}_{\text{(c)}}
    \end{align*}
Before bounding the two terms, we observe the following:
\begin{equation}
\label{eq:clique_opt}
    \sum_{t=1}^T \max_{i \in C_{I_t}} \mu_{i}(\widetilde{N}_{I_t, t}^\pi) = \sum_{C_m \in \mathcal{C}_{\mathbf{G}}} \sum_{n=1}^{N_{C_m,T}^\pi} \max_{i \in C_m} \mu_i(n).
\end{equation}
Equation~\eqref{eq:clique_opt} is a consequence of Equation~\eqref{eq:rotting_opt} (Theorem~\ref{thr:opt}), when applied to the restless bandit problems obtained by each clique when considered alone. We have for (c):

\begin{align*}
    \text{(c)} &= \sum_{t=1}^T \max_{i \in C_{I_t}} \left( \mu_{i}(\widetilde{N}_{I_t, t}^\pi)-\mu_{I_t}(\widetilde{N}_{I_t, t}^\pi) \right) \\
    &\stackrel{\text{Eq.~\eqref{eq:clique_opt}}}{=} \sum_{C_m \in \mathcal{C}_{\mathbf{G}}} \sum_{n=1}^{N_{C_m,T}^\pi} \left( \max_{i \in C_m} \mu_{i}(n)-\mu_{I_{t_{C_m,n}}}(n) \right) \\
    &\stackrel{(\Diamond)}{\le} {\color{black} 6k V_{\boldsymbol{\nu}}(T)}+ {\color{black} 4(8\alpha \sigma)^\frac{2}{3}\sum_{C_m \in \mathcal{C}_{\mathbf{G}} }\left(V_{\boldsymbol{\nu}}(T)|C_m|(N_{C_m,T}^\pi)^2\log T\right)^\frac{1}{3}} + \\
    &\qquad  + 2(2\sqrt{2}\alpha \sigma)^\frac{1}{3}\sum_{C_m \in \mathcal{C}_{\mathbf{G}} }\left(V_{\boldsymbol{\nu}}(T)^2|C_m|^2N_{C_m,T}^\pi\sqrt{\log T}\right)^\frac{1}{3}
\end{align*}
The last inequality is obtained by observing that, fixing the number of times a pull is selected, we have a nested restless bandit problem having as the time horizon the number of times the clique is pulled $N_{C_m,T}^\pi$. \texttt{RAW-UCB} plays greedily in each clique independently. Thus, when an action belonging to clique $C_m$ is selected, it is the same action that an instance of \texttt{RAW-UCB} would have played in a restless rotting bandit composed only of the actions belonging to $C_m$.\footnote{The UCB is different since the clique-specific instance of \texttt{RAW-UCB} would have used \emph{internal times} instead of the external time $t$, however, the order is preserved and the decision is the same.} In the step marked with $(\Diamond)$, this equivalence allows us to bound (c) with the summation of regret bounds of the algorithm for smaller restless bandits defined for the cliques, by using Theorem 1 from~\citep{seznec2020single} and bounding $\log N_{C_m, T}^\pi \le \log T$ and $V_{\boldsymbol{\nu}}(N_{C_m,T}^\pi) \le V_{\boldsymbol{\nu}}(T)$ for every $C_m \in \mathcal{C}_\mathbf{G}$. The last term is dominated by the other two in every quantity, and is thus omitted in the final bound.

We now focus on (b):
\begin{align*}
    \text{(b)} &= \sum_{t=1}^T (\mu_{i_t^*}(\widetilde{N}_{i_t^*, t}^*) - \max_{i \in C_{I_t}} \mu_{i}(\widetilde{N}_{I_t, t}^\pi)) \\
    &\!\!\!\!\!\!\stackrel{\text{Eq.}~\eqref{eq:overest_rotting}}{=} \sum_{C_m \in \mathcal{C}_{\mathbf{G}}} \sum_{n=1}^{N_{C_m,T}^{*}} \max_{i \in C_m} \mu_i(n) - \sum_{t=1}^T\max_{i \in C_{I_t}} \mu_{i}(\widetilde{N}_{I_t, t}^\pi) \\
    &\!\!\!\!\!\!\stackrel{\text{Eq.}~\eqref{eq:clique_opt}}{=} \sum_{C_m \in \mathcal{C}_{\mathbf{G}}} \sum_{n=1}^{N_{C_m,T}^{*}} \max_{i \in C_m} \mu_i(n) - \sum_{C_m \in \mathcal{C}_{\mathbf{G}}} \sum_{n=1}^{N_{C_m,T}^\pi} \max_{i\in C_m} \mu_i(n)
\end{align*}
The term (b) only depends on the difference between the allocation of pulls among the cliques between the optimal policy and the algorithm's policy. Thus, it makes sense to split the cliques into two sets, namely $\text{OP}$ and $\text{UP}$: the first will contain the OverPulled cliques, the second the UnderPulled cliques, which are cliques pulled by \texttt{RAW-UCB} more than the optimal policy and the cliques pulled less, respectively.
\begin{align*}
    \sum_{C_m \in \mathcal{C}_{\mathbf{G}}}\sum_{n=1}^{N_{C_m,T}^{*}} \max_{i \in C_m} \mu_i(n) &- \sum_{C_m \in \mathcal{C}_{\mathbf{G}}} \sum_{n=1}^{N_{C_m,T}^{\pi}} \max_{i \in C_m} \mu_i(n) \\ &= \sum_{C_m \in \text{UP}}\sum_{n=N_{C_m,T}^{\pi}+1}^{N_{C_m,T}^{*}}\max_{i \in C_m}\mu_i(n) -\sum_{C_m \in \text{OP}}\sum_{n=N_{C_m,T}^{*}+1}^{N_{C_m,T}^{\pi}}\max_{i \in C_m}\mu_i(n).
\end{align*}
We now introduce the auxiliary quantity $\mu_T^+(\pi)\coloneqq \max_{i \in [k]} \mu_i(\widetilde{N}_{i,T}^\pi)$. We also observe that the two terms in the RHS have the same number of addends, since the number of overpulls must be equal to the number of underpulls. Finally, we define $h_{C_m,T}$ as the number of overpulls of clique $C$.
\begin{align*}
    \sum_{C_m \in \text{UP}}\sum_{n=N_{C_m,T}^{\pi}}^{N_{C_m,T}^{*}-1}\max_{i \in C_m}\mu_i(n) &-\sum_{C_m \in \text{OP}}\sum_{n=N_{C_m,T}^{*}}^{N_{C_m,T}^{\pi}-1}\max_{i \in C_m}\mu_i(n) \\ &\le\sum_{C_m \in \text{UP}}\sum_{n=N_{C_m,T}^{\pi}}^{N_{C_m,T}^{*}-1}\mu_T^+(\pi) -\sum_{C_m \in \text{OP}}\sum_{n=N_{C_m,T}^{*}}^{N_{C_m,T}^{\pi}-1}\max_{i \in C_m}\mu_i(n)\\
    &= \sum_{C_m \in \text{OP}}\sum_{n=N_{C_m,T}^{*}}^{N_{C_m,T}^{\pi}-1}(\mu_T^+(\pi)-\max_{i \in C_m}\mu_i(n)) \\
    &=  \sum_{C_m \in \text{OP}}\sum_{h=0}^{h_{C_m,T}-1}(\mu_T^+(\pi)-\max_{i \in C_m}\mu_i(N_{C_m,T}^{*}+h)).
\end{align*}
We can now decompose the last summation by the means of events $\{\xi_t^\alpha\}_t$:
\begin{align*}
    \text{(b$_\xi$)} &\le \sum_{C_m \in \text{OP}}\sum_{h=0}^{h_{C_m,T}-1}\mathbbm{1}\{\xi_{t_{C_m,N_{C_m,T}^* + h}^\pi}^\alpha\}(\mu_T^+(\pi)-\max_{i \in C_m}\mu_i(N_{C_m,T}^{*}+h)) \\
    &\le \sum_{C_m \in \text{OP}^\xi}\sum_{h=0}^{h_{C_m,T}^\xi}(\mu_T^+(\pi)-\max_{i \in C_m}\mu_i(N_{C_m,T}^{*}+h)),
\end{align*}
where $h_{C_m,T}^\xi \coloneqq \max\{h \le h_{C_m,T} : \xi_{t_{C_m,N_{C_m,T}^* + h}^\pi}\}$ is the largest number of overpulls a clique undergoes before time $t_{C_m,N_{C_m,T}^* + h}^\pi \le T$ under the events $\xi_t^\alpha$, and $\text{OP}^\xi\coloneqq \{ C_m \in \text{OP} : h_{C_m,T}^\xi \ge 1\}$. We call, for short, $\widetilde{t}_{C_m,h}$ the time at which clique $C_m$ is overpulled for the $h$-th time i.e., $t_{C_m,N_{C_m,T}^* + h}^\pi$, and observe that:
\begin{align*}
    & \sum_{h=0}^{h_{C_m,T}^\xi}\max_{i \in C_m}\mu_i(N_{C_m,T}^{*}+h) \\
    &=\sum_{h=0}^{h_{C_m,T}^\xi}\mathbbm{1}\{h \neq h_{C_m, t_{j,N_{j,T}^\pi}} \forall j \in [k]\}\max_{i \in C_m}\mu_i(N_{C_m,T}^{*}+h) + \sum_{j \in C_m}\max_{i \in C_m}\mu_i(N_{C_m,T}^{*}+h_{C_m, t_{j,N_{j,T}^\pi}}) \\
    &= \sum_{i \in C_m} \sum_{h=0}^{h_{i,T}^\xi-1}  \mu_i(N_{C_m,\widetilde{t}_{C,h}}^{\pi}) + \sum_{j \in C_m}\max_{i \in C_m}\mu_i(\widetilde{N}_{j, t_{j, N_{j,T}^\pi}}^{\pi}) \\
    &\stackrel{\text{Eq.}~\eqref{eq:estimator_rotting}}{=} \sum_{i \in C_m} (h_{i,T}^\xi-1)\bar{\mu}_i^{h_{i,T}^\xi-1}(\widetilde{t}_{C,h_{i,T}^\xi}) + \sum_{j \in C_m}\max_{i \in C_m}\mu_i(\widetilde{N}_{j, t_{j, N_{j,T}^\pi}}^{\pi})\\
    &\stackrel{\text{Eq.}~\eqref{eq:overest_rotting}}{\ge} \sum_{i \in C_m} (h_{i,T}^\xi-1)\left(\max_{i\in [k]} \mu_i(\widetilde{N}_{i,T}^\pi) - 2c(h_{i,T}^\xi-1, \delta_{\widetilde{t}_{C,h_{i,T}^\xi}})\right)
    + \sum_{j \in C_m}\max_{i \in C_m}\mu_i(\widetilde{N}_{j, t_{j, N_{j,T}^\pi}}^{\pi}) \\
    &\ge (h_{C_m,T}^\xi-|C_m|)\max_{i\in [k]} \mu_i(\widetilde{N}_{i,T}^\pi) -2\sum_{i \in C_m} (h_{i,T}^\xi-1)c(h_{i,T}^\xi-1, \delta_T)
    + \sum_{j \in C_m}\max_{i \in C_m}\mu_i(\widetilde{N}_{j, t_{j, N_{j,T}^\pi}}^{\pi}) \\
    &= (h_{C_m,T}^\xi-|C_m|)\mu_T^+(\pi) -2\sum_{i \in C_m} (h_{i,T}^\xi-1)c(h_{i,T}^\xi-1, \delta_T)
    + \sum_{j \in C_m}\max_{i \in C_m}\mu_i(\widetilde{N}_{j, t_{j, N_{j,T}^\pi}}^{\pi}).
\end{align*}
Plugging this observation into the previous, we get:
\begin{align*}
    \text{(b$_\xi$)}  &\le \sum_{C_m \in \text{OP}^\xi}\left(|C_m|\mu_T^+(\pi)-\sum_{j \in C_m}\max_{i \in C_m}\mu_i(\widetilde{N}_{j, t_{j, N_{j,T}^\pi}}^{\pi}) + 2\sum_{i \in C_m} (h_{i,T}^\xi-1)c(h_{i,T}^\xi-1, \delta_T)\right) \\
    &= \sum_{C_m \in \text{OP}^\xi}\left(\sum_{i \in C_m}(\mu_T^+(\pi)-\max_{j\in C_m}\mu_j(\widetilde{N}_{i, t_{i, N_{i,T}^\pi}}^{\pi})) + 2\sum_{i \in C_m} (h_{i,T}^\xi-1)c(h_{i,T}^\xi-1, \delta_T)\right) \\
    &\stackrel{(\star)}{\le} 2k\sigma \sqrt{\log T} + L\sum_{C_m \in \mathcal{C}_{\mathbf{G}}} |C_m|^2 + 2\sum_{C_m \in \text{OP}^\xi}\left(\sum_{i \in C_m} (h_{i,T}^\xi-1)c(h_{i,T}^\xi-1, \delta_T)\right) \\
    &\le 2k\sigma \sqrt{\log T} + L\sum_{C_m \in \mathcal{C}_{\mathbf{G}}} |C_m|^2 + 2\sum_{C_m \in \text{OP}^\xi}\left(\sigma\sum_{i \in C_m} \sqrt{(h_{i,T}^\xi-1)\log T}\right) \\
    &\le 2k\sigma \sqrt{\log T} + L\sum_{C_m \in \mathcal{C}_{\mathbf{G}}} |C_m|^2 + 2\sum_{C_m \in \text{OP}}\left(\sigma\sqrt{\log T}\sum_{i \in C_m} \sqrt{(h_{i,T}-1)}\right) \\
    &\stackrel{(J)}{\le}  {\color{black} 2k\sigma \sqrt{\log T}} + L\sum_{C_m \in \mathcal{C}_{\mathbf{G}}} |C_m|^2 +{\color{black} 2\sum_{C_m \in \mathcal{C}_{\mathbf{G}}}\left( \sigma\sqrt{|C_m|N_{C_m,T}^\pi \log T}\right)}.
\end{align*}
The step marked with $(\star)$ is justified by the following considerations. Let $i \in C_m$, we shorten the notation for the time at which clique $C_m$ is triggered for the $(\widetilde{N}_{i,t_{i, N_{i,T}^\pi}}^\pi-m)$-th time as  $t_{i,-m} \coloneqq t_{C_m,\widetilde{N}_{i, t_{i, N_{i,T}^\pi}}^\pi-m}$. In other words, after this time, the clique $C_m$ is only chosen $m$ times before the action $i \in C_m$ is pulled for the last time. Consider the $|C_m|$ times the clique $C_m$ is chosen before pulling $i$ for the last time: then, due to the pigeonhole principle, at least one action belonging to the clique should appear at least two times before the last pull. Without loss of generality, we assume that only one action appears exactly two times, and call the first appearance time $t_{i,-m}$ and the second $t_{i,-m'}$ (note that $m'\le m\le |C_m|$). Finally, in the step marked with $(J)$ we used Jensen's inequality to find the worst allocation of overpull among the actions in the same clique, which is the uniform one, \ie $h_{i,T} \le N_{C_m,T}^\pi/|C_m|$.

We now observe that:
\begin{align*}
    \sum_{i \in C_m} & \left(\mu_T^+(\pi)-\max_{j\in C_m}\mu_j(\widetilde{N}_{i, t_{i, N_{i,T}^\pi}}^{\pi}) \right) \\
    &= \sum_{i \in C_m} \left(\mu_T^+(\pi) - \max_{j\in C_m}\left\{\mu_j(\widetilde{N}_{i, t_{i, N_{i,T}^\pi}}^{\pi}) \pm \mu_j(\widetilde{N}_{i, t_{i, N_{i,T}^\pi}}^{\pi}-m)\right\}\right) \\
    &\le \sum_{i \in C_m} \left(\mu_T^+(\pi) - \max_{j\in C_m}\mu_j(\widetilde{N}_{i, t_{i, N_{i,T}^\pi}}^{\pi}-m)+ mL\right) \\
    &\le \sum_{i \in C_m} \left(\mu_T^+(\pi) - \max_{j\in C_m}\mu_j(\widetilde{N}_{i, t_{i, N_{i,T}^\pi}}^{\pi}-m)\right) + |C_m|^2 L, 
\end{align*}
We can now prove the step $(\star)$ by bounding:
\begin{align*}
    \sum_{i \in C_m}\max_{j\in C_m}\mu_j(\widetilde{N}_{i, t_{i, N_{i,T}^\pi}}^{\pi}-m) &\ge \sum_{i \in C_m} \mu_{I_{t_{i,-m}}}(\widetilde{N}_{i, t_{i, N_{i,T}^\pi}}^{\pi}-m)\\
    &= \sum_{i \in C_m} \bar{\mu}_{I_{t_{i,-m}}}^1(t_{i,-m'}) \\
    &\stackrel{\text{Eq.}~\eqref{eq:overest_rotting}}{\ge} \sum_{i \in C_m} \left(\max_{j\in [k]}\mu_j(\widetilde{N}_{j, t_{i, -m'}}^{\pi})-2c(1, \delta_{t_{i,-m'}})\right) \\
    &\stackrel{(\dagger)}{\ge} \sum_{i \in C_m} \left(\mu_T^+(\pi)-2c(1, \delta_{T})\right),
\end{align*}
where $(\dagger)$ is a consequence of $\widetilde{N}_{j, t_{i, -m'}}^{\pi} \le  \widetilde{N}_{j, T}^{\pi}$ for every $j \in [k]$.

To conclude the proof, we need to find out what happens under $\bar{\xi_t}^\alpha$:
\begin{align*}
   \text{(b$_{\bar{\xi}}$)}  &\le \sum_{C_m \in \text{OP}}\sum_{h=0}^{h_{C_m,T}-1}\mathbbm{1}\{\bar{\xi}_{t_{C_m,N_{C_m,T}^* + h}^\pi}^\alpha\}(\mu_T^+(\pi)-\max_{i \in C_m}\mu_i(N_{C_m,T}^{*}+h)) \\
    &= \sum_{C_m \in \text{OP}} \sum_{h=0}^{h_{C_m,T}-1} \sum_{t=1}^T \mathbbm{1}\{\bar{\xi}_{t_{C_m,N_{C_m,T}^* + h}^\pi}^\alpha \wedge t_{C_m,N_{C_m,T}^* + h}^\pi = t\}(\mu_T^+(\pi)-\max_{i \in C_m}\mu_i(N_{C_m,T}^{*}+h)) \\
    &\stackrel{(*)}{\le}  \sum_{t=1}^T \mathbbm{1}\{\bar{\xi}_{t_{C_m,N_{C_m,T}^* + h}^\pi}^\alpha\}Lt \left(\sum_{C_m \in \text{OP}} \sum_{h=0}^{h_{C_m,T}-1} \mathbbm{1}\{t_{C_m,N_{C_m,T}^* + h}^\pi = t\}\right) \\
    &\le \sum_{t=1}^T \mathbbm{1}\{\bar{\xi}_{t_{C_m,N_{C_m,T}^* + h}^\pi}^\alpha\}Lt.
\end{align*}
In the step marked with $(\star)$, we use the fact that at time $t$ overpulling a clique can yield at most $Lt$ regret. The last step is a consequence that for each round $t$ we can have at most $1$ overpull. We conclude the bound by using Proposition~\ref{prop:varbound}:
\begin{align*}
    \mathbb{E}[\text{(b$_{\bar{\xi}}$)}] \le \sum_{t=1}^T \mathbb{P}\left(\bar{\xi}_t^\alpha\right)Lt\stackrel{\eqref{eq:goodevent_prob}}{\le} \sum_{t=1}^T kLt^{3-\alpha} 
    \stackrel{(\alpha \ge 5)}{\le} {\color{black} 2kL}.
\end{align*}
We then observe that, given an arbitrary concave function $g$, we have:
$$
\sum_{C_m \in \mathcal{C}_{\mathbf{G}}} g(|C_m| N_{C_m,T}^\pi) \le \sum_{C_m \in \mathcal{C}_{\mathbf{G}}} g\left(\frac{|C_m|}{k}T\right).
$$
This can be applied to component (b) with $g(\cdot) = \sqrt{\cdot}$ and to component (c) with $g(\cdot) = (\cdot)^\frac{2}{3}$.
The statement of the theorem can be obtained by summing up all the components: 
\begin{align*}
    R_{\boldsymbol{\vnu},\mathbf{G},T}(\pi) \le \mathbb{E}\left[\text{(b$_{\xi}$)}+\text{(b$_{\bar{\xi}}$)}+\text{(c)}\right].
\end{align*}
\end{proof}

\subsection{Regret Lower Bound for Rotting GTBs}

\begin{restatable}{lemma}{graph_inclusion}
\label{lemma:graph_inclusion}
Let $\mathcal{G}_n = (V_{n},E_{n})$ be a graph. Then, either the graph possesses block-diagonal connectivity and the nodes can be partitioned in $M$ disjoint cliques, i.e., $V = \bigcup_{m=1}^M C_m$, or there exist three nodes $v_1$, $v_2$ and $v_3$ such that $e_{v_1, v_2}$, $e_{v_2, v_3} \in E$ and $e_{v_1, v_3} \notin E$.
\end{restatable}
\begin{proof}
    We proceed by induction. The statement holds for $n\le 3$. We assume that $\mathcal{G}_n$ is an arbitrary graph satisfying the statement. Then we add one node $v_{n+1}$ and obtain $\mathcal{G}_{n+1} = (V_{n+1},E_{n+1})$.

    \paragraph{If $\mathcal{G}_n$ is block-diagonal connected.} 
    We list all the possible scenarios:
    \begin{itemize}
        \item If $e_{v_{n+1},i} \notin E_{n+1}$ for every $i \in V_n$, then the node $v_{n+1}$ is a single-element clique and the new graph is block-diagonal.
        \item If $e_{v_{n+1},i} \in E_{n+1}$ for every $i \in C_m$, and $e_{v_{n+1},j} \notin E_{n+1}$ for every $j \in V_n \setminus C_m$, then the node $v_{n+1}$ is added to the clique $C_m$ and the new graph is block-diagonal.
        \item If $e_{v_{n+1},i} \in E_{n+1}$ for some $i \in C_m$ and $e_{v_{n+1},j} \notin E_{n+1}$ for some $j \in C_m$, then $e_{v_{n+1}, i}$, $e_{i, j} \in E$ and $e_{v_{n+1}, j} \notin E$.
        \item If $e_{v_{n+1},i} \in E_{n+1}$ for some $i \in C_m$ and $e_{v_{n+1},j} \in E_{n+1}$ for some $j \in C_{m'}$, then $e_{v_{n+1}, i}$, $e_{v_{n+1}, j} \in E$ and $e_{i, j} \notin E$.
    \end{itemize}
    \paragraph{If there exist three nodes $v_1$, $v_2$ and $v_3$ such that $e_{v_1, v_2}$, $e_{v_2, v_3} \in E$ and $e_{v_1, v_3} \notin E$.}
    There is no way to connect $v_1$ and $v_3$ by adding a node, thus the statement still holds for $\mathcal{G}_{n+1}$.
\end{proof}

\rottingregretLBgeneric*
\begin{proof}
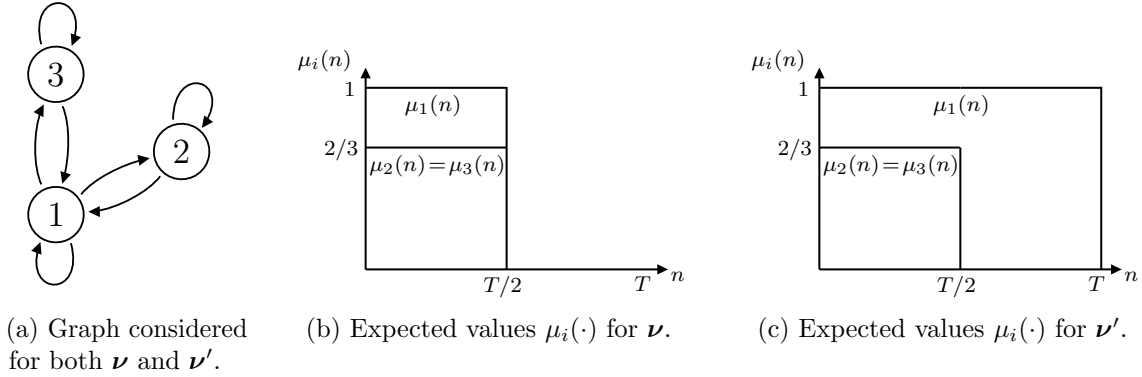
\begin{figure}[t!]
    \centering
    \subfloat[Graph considered for both $\vnu$ and $\vnu '$.]{\resizebox{0.20\linewidth}{!}{\input{img/graph_lb_rotting}}
    \label{fig:graph_lb_rotting}}
    \hfill
    \subfloat[Expected values $\mu_i( \cdot )$ for $\vnu$.]{\resizebox{0.36\linewidth}{!}{\input{img/lb_rotting_nu}} \label{fig:lb_rotting_nu}}
    \hfill
    \subfloat[Expected values $\mu_i( \cdot )$ for $\vnu '$.]{\resizebox{0.36\linewidth}{!}{\input{img/lb_rotting_nuprime}} \label{fig:lb_rotting_nuprime}}
    \caption{Instances used in the proof of Theorem~\ref{thr:rottingregretLBgeneric}.}
    \label{fig:lb_rotting}
\end{figure}
Consider the deterministic rotting scenario, \ie where $\sigma=0$. Consider two instances $\vnu$ and $\vnu '$ of $3$-armed rotting bandit with graph structure as depicted in Figure~\ref{fig:lb_rotting}. The graph which represents the connection of the arms is represented in Figure~\ref{fig:graph_lb_rotting}. The expected rewards at the different number of triggers $n$ is depicted in Figure~\ref{fig:lb_rotting_nu} for instance $\vnu$ and in Figure~\ref{fig:lb_rotting_nuprime} for instance $\vnu '$. For both instances, arms $2$ and $3$ present an expected reward equal to $2/3$ for the first $T/2$ triggers, and then the expected reward becomes $0$. On the other hand, the two instances differ in the behavior of the expected reward of arm $1$. Indeed, such reward is $1$ until we trigger the arm $T/2$ times for instance $\vnu$ and for all the $T$ triggers for instance $\vnu '$. 

We recall that the clairvoyant is aware of both the graph $\mathbf{G}$ and the expected values $\mu_i(n)$, for every $i\in [k]$ and $n \in [T]$. We can easily compute the total reward for the best policy possible $\pi^*$ for instance $\vnu$:
\begin{align}
    J_{\boldsymbol{\vnu },\mathbf{G},T}(\pi^*) &= \frac{2}{3} T, \label{eq:lbrotting:jstarnu} 
\end{align}
which corresponds to pulling arms $2$ and $3$ only (both for $T/2$ times), and for instance $\vnu '$:
\begin{align}
    J_{\boldsymbol{\vnu '},\mathbf{G},T}(\pi^*) &= T, \label{eq:lbrotting:jstarnuprime}
\end{align}
which corresponds to always pull arm $1$ for $T$ times. We highlight that the optimal policy $\pi^*$ is different for the two instances.

We now need to introduce some additional notations that will be used in the proof. We call $\mathbb{E}_{\vnu} \left[ N_i^R \left( n \right) \right]$ the expected number of pulls for arm $i$ generating reward (\ie for which the expected reward is different from $0$) up to time $n$ for instance $\vnu$. We now start by observing that, up to the round $T/2$, the two instances are exactly the same, so every policy $\pi$ will have the same behavior in expectation. Given that, we observe that for both instances we have the same reward, equal to:
\begin{align*}
J_{\boldsymbol{\vnu},\mathbf{G},T/2}(\pi) = J_{\boldsymbol{\vnu '},\mathbf{G},T/2}(\pi) &= \Nonehalf + \frac{2}{3} \Ntwohalf + \frac{2}{3} \Nthreehalf \\ 
&= \frac{T}{2} - \frac{1}{3} \Ntwohalf - \frac{1}{3} \Nthreehalf,
\end{align*}
where the last equality follows from $\Nonehalf + \Ntwohalf + \Nthreehalf = T/2$.
This result is valid for both $\vnu$ and $\vnu '$, as the policy will behave in the same way, and so $\mathbb{E}_{\vnu} \left[ N_i^R \left( \frac{T}{2} \right) \right] = \mathbb{E}_{\vnu '} \left[ N_i^R \left( \frac{T}{2} \right) \right]$, for every $i \in [ 3 ]$, as the policy on the two instances $\vnu$ and $\vnu '$ are not distinguishable the first $T/2$ rounds.

We now have to understand what will happen from $T/2$ to $T$ in the best case possible.

\noindent\textbf{Instance $\vnu$.}~~We can easily see how for arm $1$ we have terminated the pulls which generate reward, so we have to pull arms $2$ and $3$. We can now compute the remaining triggers generating reward for arm $2$ in the second half of the rounds:
\begin{align*}
    \NtwoT - \Ntwohalf & \leq \underbrace{\frac{T}{2}}_{\substack{\text{Triggers} \\ \text{initially} \\  \text{available}}} - \underbrace{\Ntwohalf}_{\substack{\text{Already} \\ \text{used}}} \\ 
    & \qquad - \underbrace{\left( \frac{T}{2} - \Ntwohalf - \Nthreehalf \right)}_{\substack{\text{Triggers used} \\ \text{from arm 1}}} \\ 
    & \leq \Nthreehalf
\end{align*}
We can do the same reasoning for arm $3$ and, for symmetry, we get:
\begin{align*}
    \NthreeT - \Nthreehalf \leq \Ntwohalf
\end{align*}
We now consider a policy using all these triggers, and we compute the expected cumulative reward: 
\begin{align*}
J_{\boldsymbol{\vnu},\mathbf{G},T}&(\pi) \\
& \leq J_{\boldsymbol{\vnu},\mathbf{G},T/2}(\pi) + \frac{2}{3} \Ntwohalf + \frac{2}{3} \Nthreehalf \\ 
& \leq \frac{T}{2} - \frac{1}{3} \Ntwohalf - \frac{1}{3} \Nthreehalf  + \frac{2}{3} \Ntwohalf + \frac{2}{3} \Nthreehalf \\
& \leq \frac{T}{2} + \frac{1}{3} \Ntwohalf + \frac{1}{3} \Nthreehalf .
\end{align*}

\noindent\textbf{Instance $\vnu '$.}~~Instead, for instance $\vnu '$, we can easily see that the best choice from $T/2$ to $T$ is to always pull arm $1$ for all the $T/2$ rounds, receiving a reward of $1$ each time. Given that, we have:
\begin{align*}
J_{\boldsymbol{\vnu '},\mathbf{G},T}(\pi) & \leq J_{\boldsymbol{\vnu '},\mathbf{G},T/2}(\pi) + \frac{T}{2} \\ 
& = T - \frac{1}{3} \Ntwohalf - \frac{1}{3} \Nthreehalf.
\end{align*}

\noindent\textbf{Regret.}~~Moving to the regret, we have for instance $\vnu$:
\begin{align*}
R_{\boldsymbol{\vnu},\mathbf{G},T}(\pi) &= J_{\boldsymbol{\vnu },\mathbf{G},T}(\pi^*) - J_{\boldsymbol{\vnu },\mathbf{G},T}(\pi) \\ 
&\geq \frac{2}{3}T - \frac{T}{2} - \frac{1}{3} \Ntwohalf - \frac{1}{3} \Nthreehalf \\
&\geq \frac{T}{6} - \frac{1}{3} \Ntwohalf - \frac{1}{3} \Nthreehalf ,
\end{align*}
while for instance $\vnu '$: 
\begin{align*}
R_{\boldsymbol{\vnu '},\mathbf{G},T}(\pi) &= J_{\boldsymbol{\vnu '},\mathbf{G},T}(\pi^*) - J_{\boldsymbol{\vnu '},\mathbf{G},T}(\pi) \\ 
&\geq \frac{1}{3} \Ntwohalf + \frac{1}{3} \Nthreehalf .
\end{align*}
We can now compute a lower bound on the regret:
\begin{align*}
    R_{T}(\mathfrak{A}) &= \max \left\{ R_{\boldsymbol{\vnu},\mathbf{G},T}(\pi) , R_{\boldsymbol{\vnu '},\mathbf{G},T}(\pi) \right\} \\ 
    &\geq \frac{1}{2} \left( R_{\boldsymbol{\vnu},\mathbf{G},T}(\pi) + R_{\boldsymbol{\vnu '},\mathbf{G},T}(\pi) \right) \\ 
    &= \frac{1}{2} \left( \frac{T}{6} - \frac{1}{3} \Ntwohalf - \frac{1}{3} \Nthreehalf \right. \\ 
    & \qquad \qquad \left. + \frac{1}{3} \Ntwohalf + \frac{1}{3} \Nthreehalf \right) \\ 
    &= \frac{T}{12}.
\end{align*}
This proof holds for the specific graph structure we discussed here. However, by joining this result with the one of Lemma~\ref{lemma:graph_inclusion}, we can generalize this result for every non-block-diagonal connectivity matrix.
\end{proof}

%% file: img/graph_lb_rotting.tex
\tikzset{every picture/.style={line width=0.75pt}} 

\begin{tikzpicture}[x=0.75pt,y=0.75pt,yscale=-1,xscale=1]

\draw   (65.72,66.33) .. controls (70.08,59.24) and (79.36,57.03) .. (86.45,61.39) .. controls (93.54,65.75) and (95.75,75.03) .. (91.39,82.12) .. controls (87.03,89.21) and (77.75,91.43) .. (70.66,87.07) .. controls (63.57,82.71) and (61.36,73.42) .. (65.72,66.33) -- cycle ;
\draw   (134,108.3) .. controls (138.36,101.21) and (147.65,99) .. (154.74,103.36) .. controls (161.83,107.72) and (164.04,117) .. (159.68,124.09) .. controls (155.32,131.18) and (146.04,133.39) .. (138.95,129.03) .. controls (131.86,124.67) and (129.64,115.39) .. (134,108.3) -- cycle ;
\draw   (65.75,142.6) .. controls (70.11,135.51) and (79.4,133.3) .. (86.49,137.66) .. controls (93.58,142.02) and (95.79,151.3) .. (91.43,158.39) .. controls (87.07,165.48) and (77.79,167.7) .. (70.7,163.34) .. controls (63.61,158.98) and (61.39,149.69) .. (65.75,142.6) -- cycle ;
\draw    (135.2,129.66) .. controls (128.76,138.11) and (111.83,147.29) .. (99.16,148.74) ;
\draw [shift={(96.17,148.93)}, rotate = 359.31] [fill={rgb, 255:red, 0; green, 0; blue, 0 }  ][line width=0.08]  [draw opacity=0] (6.25,-3) -- (0,0) -- (6.25,3) -- cycle    ;
\draw    (92.26,139.48) .. controls (98.96,131.13) and (114.26,123.41) .. (126.75,120.4) ;
\draw [shift={(129.52,119.8)}, rotate = 169.23] [fill={rgb, 255:red, 0; green, 0; blue, 0 }  ][line width=0.08]  [draw opacity=0] (6.25,-3) -- (0,0) -- (6.25,3) -- cycle    ;
\draw    (70.65,134.07) .. controls (66.22,124.41) and (66.09,105.16) .. (70.76,93.3) ;
\draw [shift={(71.99,90.57)}, rotate = 117.36] [fill={rgb, 255:red, 0; green, 0; blue, 0 }  ][line width=0.08]  [draw opacity=0] (6.25,-3) -- (0,0) -- (6.25,3) -- cycle    ;
\draw    (82.18,91.56) .. controls (86.4,101.4) and (86.02,118.53) .. (82.8,130.97) ;
\draw [shift={(82.03,133.7)}, rotate = 287.28] [fill={rgb, 255:red, 0; green, 0; blue, 0 }  ][line width=0.08]  [draw opacity=0] (6.25,-3) -- (0,0) -- (6.25,3) -- cycle    ;
\draw    (86.71,166.04) .. controls (94.73,196.05) and (60.97,195.47) .. (69.27,168) ;
\draw [shift={(70.16,165.37)}, rotate = 110.64] [fill={rgb, 255:red, 0; green, 0; blue, 0 }  ][line width=0.08]  [draw opacity=0] (6.25,-3) -- (0,0) -- (6.25,3) -- cycle    ;

\draw    (142.2,98.87) .. controls (142.56,67.8) and (174.78,76.91) .. (159.39,101.1) ;
\draw [shift={(157.82,103.39)}, rotate = 306.25] [fill={rgb, 255:red, 0; green, 0; blue, 0 }  ][line width=0.08]  [draw opacity=0] (6.25,-3) -- (0,0) -- (6.25,3) -- cycle    ;
\draw    (70.58,57.88) .. controls (63.03,27.75) and (96.77,28.85) .. (88.05,56.19) ;
\draw [shift={(87.12,58.81)}, rotate = 291.52] [fill={rgb, 255:red, 0; green, 0; blue, 0 }  ][line width=0.08]  [draw opacity=0] (6.25,-3) -- (0,0) -- (6.25,3) -- cycle    ;

\draw (72.4,66.6) node [anchor=north west][inner sep=0.75pt]  [font=\Large] [align=left] {$\displaystyle 3$};
\draw (140.5,108.5) node [anchor=north west][inner sep=0.75pt][font=\Large]    [align=left] {$\displaystyle 2$};
\draw (72.5,143) node [anchor=north west][inner sep=0.75pt][font=\Large]    [align=left] {$\displaystyle 1$};

\end{tikzpicture}

%% file: img/lb_rotting_nu.tex
\tikzset{every picture/.style={line width=0.75pt}} 

\begin{tikzpicture}[x=0.75pt,y=0.75pt,yscale=-1,xscale=1]

\draw    (70.14,230.31) -- (70.14,133.15) ;
\draw [shift={(70.14,130.15)}, rotate = 90] [fill={rgb, 255:red, 0; green, 0; blue, 0 }  ][line width=0.08]  [draw opacity=0] (5.36,-2.57) -- (0,0) -- (5.36,2.57) -- cycle    ;
\draw    (70.14,230.31) -- (217.08,230.31) ;
\draw [shift={(220.08,230.31)}, rotate = 180] [fill={rgb, 255:red, 0; green, 0; blue, 0 }  ][line width=0.08]  [draw opacity=0] (5.36,-2.57) -- (0,0) -- (5.36,2.57) -- cycle    ;
\draw    (70.44,140.21) -- (140.04,140.21) ;
\draw    (70.32,169.85) -- (139.77,169.85) ;
\draw    (140.04,230.31) -- (140.04,139.84) ;

\draw (127,232) node [anchor=north west][inner sep=0.75pt]  [font=\scriptsize]  {$T/2$};
\draw (202,232) node [anchor=north west][inner sep=0.75pt]  [font=\scriptsize]  {$T$};
\draw (220,230) node [anchor=north west][inner sep=0.75pt]  [font=\scriptsize]  {$n$};
\draw (35,120) node [anchor=north west][inner sep=0.75pt]  [font=\scriptsize]  {$\mu _{i}( n)$};
\draw (58,136) node [anchor=north west][inner sep=0.75pt]  [font=\scriptsize]  {$1$};
\draw (48,164) node [anchor=north west][inner sep=0.75pt]  [font=\scriptsize]  {$2/3$};

\draw (70.5,172) node [anchor=north west][inner sep=0.75pt]  [font=\scriptsize]  {$\mu_{2}(n) \! = \! \mu_{3}(n)$};
\draw (88,142) node [anchor=north west][inner sep=0.75pt]  [font=\scriptsize]  {$\mu_{1}(n)$};

\end{tikzpicture}

%% file: img/lb_rotting_nuprime.tex
\tikzset{every picture/.style={line width=0.75pt}} 

\begin{tikzpicture}[x=0.75pt,y=0.75pt,yscale=-1,xscale=1]

\draw    (70.14,230.31) -- (70.14,133.15) ;
\draw [shift={(70.14,130.15)}, rotate = 90] [fill={rgb, 255:red, 0; green, 0; blue, 0 }  ][line width=0.08]  [draw opacity=0] (5.36,-2.57) -- (0,0) -- (5.36,2.57) -- cycle    ;
\draw    (70.14,230.31) -- (217.08,230.31) ;
\draw [shift={(220.08,230.31)}, rotate = 180] [fill={rgb, 255:red, 0; green, 0; blue, 0 }  ][line width=0.08]  [draw opacity=0] (5.36,-2.57) -- (0,0) -- (5.36,2.57) -- cycle    ;
\draw    (70.44,140.21) -- (140.04,140.21) ;
\draw    (70.32,169.85) -- (139.77,169.85) ;
\draw    (140.04,230.31) -- (140.04,169.84) ;
\draw    (209.89,230.16) -- (209.89,139.69) ;
\draw    (140.04,140.21) -- (209.65,140.21) ;

\draw (127,232) node [anchor=north west][inner sep=0.75pt]  [font=\scriptsize]  {$T/2$};
\draw (202,232) node [anchor=north west][inner sep=0.75pt]  [font=\scriptsize]  {$T$};
\draw (220,230) node [anchor=north west][inner sep=0.75pt]  [font=\scriptsize]  {$n$};
\draw (35,120) node [anchor=north west][inner sep=0.75pt]  [font=\scriptsize]  {$\mu _{i}( n)$};
\draw (58,136) node [anchor=north west][inner sep=0.75pt]  [font=\scriptsize]  {$1$};
\draw (48,164) node [anchor=north west][inner sep=0.75pt]  [font=\scriptsize]  {$2/3$};

\draw (70.5,172) node [anchor=north west][inner sep=0.75pt]  [font=\scriptsize]  {$\mu_{2}(n) \! = \! \mu_{3}(n)$};
\draw (125,142) node [anchor=north west][inner sep=0.75pt]  [font=\scriptsize]  {$\mu_{1}(n)$};

\end{tikzpicture}